\documentclass{article}

\usepackage[main,preprint]{neurips_2026} 

\usepackage[utf8]{inputenc} 
\usepackage[T1]{fontenc}    
\usepackage{hyperref}       
\usepackage{url}            
\usepackage{booktabs}       
\usepackage{amsfonts}       
\usepackage{nicefrac}       
\usepackage{microtype}      
\usepackage{xcolor}         
\usepackage{amsmath, amssymb}
\usepackage{graphicx}
\usepackage{enumitem}
\usepackage{multirow}
\usepackage{array}
\usepackage{amsthm}
\usepackage{longtable}
\usepackage{tabularx}
\newtheorem{proposition}{Proposition}

\theoremstyle{definition}

\theoremstyle{remark}

\title{\textit{Ex Ante} Evaluation of AI-Induced Idea Diversity Collapse}

\author{%
  Nafis Saami Azad \\
  Bellini College of Artificial Intelligence, Cybersecurity, and Computing\\
  University of South Florida\\
  Tampa, FL--33620, USA \\
  \texttt{nafisazad@usf.edu} \\
  \And
  Raiyan Abdul Baten\\
  Bellini College of Artificial Intelligence, Cybersecurity, and Computing\\
  University of South Florida\\
  Tampa, FL--33620, USA \\
  \texttt{rbaten@usf.edu} \\
}

\begin{document}

\maketitle

\begin{abstract}
Creative AI systems are typically evaluated at the level of individual utility, yet creative outputs are consumed in populations: an idea loses value when many others produce similar ones. This creates an evaluation blind spot, as AI can improve individual outputs while increasing population-level crowding. We introduce a human-relative framework for benchmarking AI-induced human diversity collapse without requiring human--AI interaction data, providing an \textit{ex ante} protocol to estimate crowding risk from model-only generations and matched unaided human baselines. By modeling ideas as congestible resources, we show that source-level crowding is identifiable from within-distribution comparisons, yielding an excess-crowding coefficient $\Delta$ and a human-relative diversity ratio $\rho$. We show that $\rho\ge1$ is the no-excess-crowding parity condition and connect $\Delta$ to an adoption game with exposure-dependent redundancy costs. Across short stories, marketing slogans, and alternative-uses tasks, three frontier LLMs fall below parity across crowding kernels. Estimates stabilize with feasible model-only sample sizes. Importantly, generation-protocol variants show that crowding can be reduced through targeted design, making diversity collapse an actionable, development-time evaluation target for population-aware creative AI.
\end{abstract}

\section{Introduction} 
Generative AI systems are increasingly used to support human creativity, from writing and design to scientific ideation. Their effectiveness is typically evaluated at the level of \textit{individual} utility: a model is deemed successful if it helps a person write better stories, generate more ideas, or produce more compelling artifacts. However, in creative domains, outputs are evaluated and consumed in \textit{populations}, and the value of a creative product depends not only on its intrinsic quality, but also on its uniqueness relative to other products in the population~\cite{amabile1982social,acar2014assessing,bangash2025musescorer}. When many users rely on the same generative model for inspiration, their final outputs can become correlated, leading to \textit{AI-induced human diversity collapse}: a reduction in the collective diversity of human-generated ideas~\cite{doshi2024generative,anderson2024homogenization,kumar2025human}. This creates an evaluation blind spot that individual-level approaches do not capture: a system can improve each user's expected output while making the population of human outputs more redundant.

Current empirical evidence on AI-induced diversity collapse relies on \textit{post hoc evaluation}: researchers collect human--AI co-created outputs and measure idea diversity after the fact~\cite{doshi2024generative,anderson2024homogenization,kumar2025human}. Such studies are essential for documenting deployed effects, but this evaluation paradigm is poorly suited as a development-time, population-aware benchmark. It requires expensive, time-intensive human-subject experiments for each model, task, and interface condition; it does not provide a model-side quantity that can be monitored, compared, or optimized before deployment; and it leaves users without a clear way to reason about whether one model may help their final ideas stand out more than another.

To close this gap, we introduce a human-relative evaluation framework for benchmarking AI-induced human diversity collapse without requiring human--AI interaction data. The framework formalizes the population-level claim being evaluated, the proper counterfactual to contextualize it, and the evidence needed to estimate it. It uses two observable ingredients: model-only generations from a fixed model-condition and matched unaided human baselines from the same task condition. 

Our approach is based on two key observations. \textit{First}, ideas are not ordinary goods when they function as \textit{sources of inspiration}. Inspiration sources behave like shared, congestible resources: \textit{many individuals cannot draw from the same idea without diminishing its value}, because repeated use increases overlap in downstream idea-space~\cite{baten2020creativity,baten2024ai,baten2022novel,kelty2025innovation,ahmed2026semantic}. This perspective gives rise to a population game in which generative models act as shared source distributions that shape the human outputs users ultimately produce. \textit{Second}, crowding in idea-space must be contextualized to be meaningful. Even without AI, independently generated human ideas exhibit a task-specific level of overlap~\cite{acar2014assessing,forthmann2020scrutinizing,bangash2025musescorer}. We therefore use human crowding as the baseline and measure the extent to which model-only generations introduce \textit{excess crowding} beyond ordinary human convergence.

A central consequence is that population-level crowding effects can be estimated from \textit{within-distribution} comparisons at the source: human--human crowding estimates the baseline level of task-specific convergence, while model-only crowding estimates the concentration induced by repeated draws from the same model-condition. Their comparison yields an excess-crowding coefficient that links source-level crowding to the population-level adoption game, while the human-relative ratio defines a parity condition: a model-condition introduces no excess crowding exactly when its outputs are no more crowded than the matched human baseline. Since this coefficient drives the population-level redundancy cost, it can be estimated before human--AI outputs are collected. This enables \textit{ex ante} evaluation of AI-induced human diversity collapse: before running a new human--AI study, developers can use model-only sampling to estimate a model-condition's population-level crowding risk. The resulting quantities also admit a decision-theoretic interpretation, characterizing when AI adoption is individually rational as a function of expected crowding, value of distinctiveness, and beliefs about others' behavior.

We instantiate the framework on three creative task families---short stories, marketing slogans, and alternative uses of common objects---and evaluate three frontier LLMs under matched task instructions. Across the primary crowding kernel, all evaluated neutral model-conditions fall below human-relative parity, indicating positive excess crowding relative to matched unaided human baselines. We then stress-test the framework along three dimensions. First, rarefaction diagnostics show that the estimated pairwise crowding coefficients stabilize with feasible model-only sample sizes, supporting the practicality of development-time evaluation. Second, task-specific crowding kernels show that the results are not artifacts of a single representation: story crowding persists under plot-synopsis similarity, slogan crowding persists under lexical-template overlap, and alternative-uses crowding persists under concept-bucket co-membership. Third, generation-protocol variants show that crowding is not fixed by the base model alone: interventions such as temperature scaling and persona-mixture prompting can move model-conditions toward human-relative parity. Together, these results reframe diversity collapse from a post hoc diagnosis into an actionable evaluation target. For AI developers, the framework provides a way to audit and reduce population-level crowding before deployment; for users and downstream evaluators, it clarifies why the value of AI assistance depends not only on private quality gains, but also on how many others draw from the same generative source.

\section{Related Work}
\subsection{Post Hoc Evaluation of AI-Assisted Human Creativity}
The most direct evidence for AI-induced human diversity collapse comes from studies of AI-assisted creativity. In these experiments, participants generate stories, ideas, designs, or other artifacts with or without AI assistance, and researchers evaluate the resulting human outputs. This paradigm has shown that AI assistance can improve individual-level outcomes while narrowing the distribution of outputs across users~\cite{anderson2024homogenization,doshi2024generative}. Similar human--AI co-creation studies evaluate whether AI-generated suggestions, feedback, or writing support change the novelty or usefulness of human creative products~\cite{chakrabarty2024art,kumar2025human,ashkinaze2025ai,sourati2026homogenizing}. Methodologically, these studies use combinations of human ratings, expert judgments, semantic similarity, clustering, lexical overlap, and idea-level diversity measures to compare human-only and AI-assisted output distributions~\cite{anderson2024homogenization,doshi2024generative,chakrabarty2024art,kumar2025human}. This paradigm is essential because it observes the realized human outputs in a particular task, model, interface, and interaction protocol. Its limitation is that it is expensive, system-specific, and retrospective: each new model-condition requires new human--AI interaction data before its population-level effect can be assessed.

The measurement tools used in these studies draw on a long tradition of human creativity assessment in psychology. Human responses to common divergent-thinking tasks, such as the Alternative Uses Task (AUT), are typically evaluated using three complementary approaches~\cite{beketayev2016scoring,runco2012standard,dumas2014understanding,runco1992scoring}. \textit{Idea-intrinsic} measures assess properties of a response or response set, such as semantic distance from a prompt, within-person semantic diversity, or elaboration~\cite{guilford1967nature,torrance1966nurture,runco2012standard,beaty2021automating,snyder2004creativity,buczak2023machines}. \textit{Social-rarity} measures score an idea by its infrequency in a reference population, often requiring rephrasings of the same underlying concept to be grouped into shared buckets~\cite{acar2014assessing,forthmann2020scrutinizing,reiter2019scoring,buczak2023machines,baten2021cues,olson2021naming,deyoung2008cognitive}. \textit{Subjective-rating} approaches, including the Consensual Assessment Technique, ask human judges to rate novelty, usefulness, or quality directly~\cite{baten2020creativity,amabile1982social}. Recent computational work has scaled these traditions using embeddings, clustering, supervised scoring, and retrieval-assisted LLM bucketing~\cite{beaty2021automating,organisciak2023beyond,organisciak2020open,bangash2025musescorer,snyder2004creativity,bossomaier2009semantic,dumas2021measuring,stevenson2020automated}. This literature provides mature ways to measure creative quality and population-relative distinctiveness once outputs are observed; the open challenge is how to convert such measurements into development-time benchmarks for AI systems before the corresponding human--AI outputs exist.

\subsection{From Model-Output Diversity to Human-Relative Evaluation}
A parallel literature evaluates the creative outputs of generative models themselves~\cite{smith2025comprehensive,wenger2025we}. Work on story generation, divergent-thinking tasks, and model-based creative ideation commonly measures model outputs using n-gram diversity, pairwise semantic distance, novelty, surprise, lexical or syntactic complexity, association-based metrics, and human or LLM judgments~\cite{zhang2025noveltybench,ismayilzada2024evaluating,lin2024evaluating,lin2023llm,chakrabarty2026can,lu2024ai}. Recent benchmark and survey work emphasizes that LLM creativity evaluation remains fragmented across tasks and metrics, with separate measures for quality, novelty, diversity, and domain-specific creative behavior~\cite{qiu2025deep,hou2025creativityprism,lin2025creativity}. These evaluations often show that model outputs can be fluent, coherent, and high quality while still converging on repeated themes, templates, or semantic regions~\cite{xu2025echoes,ismayilzada2024evaluating,hou2025creativityprism,lin2025creativity,zhang2025noveltybench,shypula2025evaluating}. Such convergence is expected from a generation perspective: LLMs sample from prompt-conditioned distributions shaped by shared training data, post-training objectives, decoding parameters, and prompting protocols~\cite{shypula2025evaluating,holtzman2019curious,kirk2023understanding,nguyen2024turning,yun2025price}. When many generations are elicited from similar prompts under similar settings, repeated tropes, phrasings, and high-probability continuations are natural unless the generation process is explicitly diversified~\cite{li2025jointly,wan2026diverse}. However, ordinary AI-output diversity metrics do not provide the human counterfactual needed to contextualize AI-induced human diversity collapse. They can say whether model outputs are more or less diverse, but not whether shared model use adds excess crowding beyond ordinary human convergence, or whether that excess crowding implies an adoption-dependent, population-level redundancy cost. Our framework addresses this gap directly.

\section{A Human-Relative Theory of Idea-Space Crowding}
\label{sec:theory}
\subsection{Human-Relative Crowding}
Consider a creative task condition $k$ with output space $\mathcal{Y}_k$. Let $H_k$ denote the distribution of unaided human outputs over $\mathcal{Y}_k$. Likewise, let $A_{m,k}$ denote the distribution induced by model $m$ under a fixed generation protocol. We use \textit{model-only sampling} to refer to repeated independent generations from this same model-condition, without human interaction. Let $K_k:\mathcal{Y}_k\times\mathcal{Y}_k\to[0,1]$ be a crowding kernel, where larger values indicate that two outputs occupy a more similar region of idea-space. The kernel is task-dependent: it may encode semantic similarity, shared plot structure, lexical-template overlap, or co-membership in the same concept bucket, as appropriate. The formal theory requires only that the same $K_k$ be applied to both human and model-only data, be bounded in $[0,1]$, and be interpretable as pairwise crowding.

We define unaided human crowding and model-only crowding for task condition $k$ as
\begin{align}
\kappa^H_k &= \mathbb{E}_{h,h'\sim H_k}[K_k(h,h')],
&
\kappa^A_{m,k} &= \mathbb{E}_{a,a'\sim A_{m,k}}[K_k(a,a')],
\label{eq:kappas}
\end{align}
where both expectations are over independent draws. The excess crowding coefficient is
\begin{equation}
\Delta_{m,k}=\max\{0,\kappa^A_{m,k}-\kappa^H_k\}.
\label{eq:delta}
\end{equation}
This coefficient measures the excess concentration of the model-conditioned source distribution relative to the unaided human counterfactual. We also define the human-relative diversity ratio
\begin{equation}
\rho_{m,k}=\frac{1-\kappa^A_{m,k}}{1-\kappa^H_k},
\qquad \kappa^H_k<1.
\label{eq:rho}
\end{equation}
The ratio normalizes model-only diversity by the matched human baseline for the given task condition $k$. This normalization is necessary because creative tasks differ in their baseline level of convergence: humans also reuse tropes, affordances, and templates. A human-relative counterfactual therefore prevents intrinsically constrained tasks from being mistaken for model-specific concentration.

Importantly, both quantities are identifiable from within-distribution samples: pairs of unaided human outputs estimate \(\kappa^H_k\), and pairs of model-only generations estimate \(\kappa^A_{m,k}\). These matched human-only and model-only samples are therefore sufficient to estimate excess crowding at the source-distribution level, without observing realized human--AI interactions. The next subsection shows how this source-level quantity enters a population adoption game.
\begin{proposition}[Human-relative parity is the no-externality condition]
\label{prop:parity}
If $\kappa^H_k<1$, then
\begin{equation}
\Delta_{m,k}=0
\quad\Longleftrightarrow\quad
\rho_{m,k}\ge 1.
\label{eq:parity_equivalence}
\end{equation}
\end{proposition}
This proposition makes $\rho=1$ a benchmark threshold: it is exactly the point at which the model-conditioned source distribution introduces no excess crowding beyond ordinary human convergence under the same task constraints. Proofs of all propositions are provided in Appendix~\ref{app:proofs}.

\subsection{Mean-field Approximation of a Congestion Game: From Crowding to Adoption}
The coefficient $\Delta_{m,k}$ is a pairwise property of a model-task distribution, representing the per-pair excess-crowding hazard. Population cost depends on how many other creators draw from that same distribution. Let $X_{-i}$ denote the number of other creators using model $m$ in task condition $k$. Each additional adopter creates another opportunity for the focal creator's output to overlap with others.

Under an independent-exposure approximation, the probability of no excess redundancy encounter is approximately $\exp\{-X_{-i}\Delta_{m,k}\}$. The probability of at least one excess encounter is therefore $1-\exp\{-X_{-i}\Delta_{m,k}\}$. Multiplying by the value of distinctiveness in task $k$, $\gamma_k\ge0$, we get the redundancy cost
\begin{equation}
C_{m,k}(X_{-i})
=
\gamma_k
\left(
1-\exp\{-X_{-i}\Delta_{m,k}\}
\right).
\label{eq:redundancy_cost}
\end{equation}
The cost saturates because once an output is already non-distinct, additional overlaps add less marginal damage. This is not the only possible choice for the redundancy cost function; we use the exponential form as a monotone saturating exposure model that cleanly separates the pairwise excess crowding $\Delta_{m,k}$ from population exposure $X_{-i}$.

Let $R_{i,k}$ be creator $i$'s expected payoff from unaided production, and let $Q_{i,m,k}$ be creator $i$'s expected payoff from using model $m$ before redundancy costs. The private AI advantage is
\begin{equation}
B_{i,m,k}=Q_{i,m,k}-R_{i,k}.
\label{eq:private_advantage}
\end{equation}
A rational creator adopts AI when
\begin{equation}
B_{i,m,k}>C_{m,k}(X_{-i}).
\label{eq:adoption_condition}
\end{equation}
The human-relative diversity ratio $\rho_{m,k}$ affects this decision through the excess-crowding coefficient $\Delta_{m,k}$, which determines the redundancy cost. Combining Eqs.~\ref{eq:delta} and~\ref{eq:rho} gives
\begin{equation}
\Delta_{m,k}
=
\max\{0,(1-\rho_{m,k})(1-\kappa^H_k)\}.
\label{eq:delta_from_rho}
\end{equation}
Thus, $\rho_{m,k}\ge1$ implies $\Delta_{m,k}=0$, while lower values of $\rho_{m,k}<1$ imply larger excess crowding. Therefore, lower human-relative diversity raises the redundancy cost $C_{m,k}(X_{-i})$ and increases the private benefit $B_{i,m,k}$ required for rational adoption.

\begin{proposition}[Critical benefit threshold]
\label{prop:critical_benefit}
For exposure level $X_{-i}$, using model $m$ is rational iff
\begin{equation}
B_{i,m,k}>
B^{\mathrm{crit}}_{m,k}(X_{-i})
=
\gamma_k\left(1-\exp\{-X_{-i}\Delta_{m,k}\}\right).
\label{eq:bcrit}
\end{equation}
For $\rho_{m,k}<1$, this threshold is decreasing in $\rho_{m,k}$ and increasing in $X_{-i}$, $\gamma_k$, and $\Delta_{m,k}$.
\end{proposition}
The implication is intuitive: a below-parity model can still be rational to use when exposure is low, distinctiveness is unimportant, or private benefit is large. But in a high-adoption creative market, low $\rho$ imposes a larger distinctiveness discount.

\subsection{Population Exposure}
Now consider a population of $N$ creators facing the same task condition $k$. For a focal creator $i$, suppose each of the other $N-1$ creators independently adopts the same model-condition with probability $p$. Then the number of other adopters is
$X_{-i}\sim \mathrm{Binomial}(N-1,p)$. Taking the expectation of Eq.~\ref{eq:redundancy_cost} over this adoption process gives
\begin{equation}
\mathbb{E}[C_{m,k}]
=
\gamma_k
\left[
1-
\left(
1-p+p\exp\{-\Delta_{m,k}\}
\right)^{N-1}
\right].
\label{eq:expected_cost}
\end{equation}
The derivation follows by taking the binomial expectation over $X_{-i}$; see Appendix~\ref{app:proofs}. This expression separates the empirically benchmarked model property, $\Delta_{m,k}$, from the population context, $p$ and $N$, and the task value of distinctiveness, $\gamma_k$. Thus, the same model can impose different expected redundancy costs depending on adoption prevalence and market size. 

The limiting case of mass adoption clarifies why the human-parity boundary matters.
\begin{proposition}[Mass-adoption limit]
\label{prop:mass_limit}
For the realized-exposure cost in Eq.~\ref{eq:redundancy_cost},
\begin{equation}
\lim_{X_{-i}\to\infty}C_{m,k}(X_{-i})
=
\begin{cases}
0, & \rho_{m,k}\ge1,\\
\gamma_k, & \rho_{m,k}<1.
\end{cases}
\label{eq:mass_limit}
\end{equation}
\end{proposition}
If $\rho_{m,k}\ge1$, shared model use adds no excess crowding cost at any exposure level. If $\rho_{m,k}<1$, source-level excess crowding compounds with adoption, reaching the full distinctiveness penalty $\gamma_k$ in the mass-adoption limit. Thus, although $\rho$ is estimated from matched human-only and model-only data, it has a population-level interpretation: it determines whether shared model use creates an excess crowding externality in human outputs, making human-relative parity a meaningful benchmark target.

\section{Benchmark Design}
\label{sec:benchmark}
The theory identifies $\Delta_{m,k}$ as the pairwise excess-crowding coefficient that enters the adoption game. Empirically, the benchmark estimates $\kappa^H_k$, $\kappa^A_{m,k}$, $\Delta_{m,k}$, and $\rho_{m,k}$ from matched human and model-generated samples. Rather than construct a general-purpose model leaderboard, we instantiate a \textit{task-relative} evaluation procedure using a human baseline, a model-conditioned sample, and a task-appropriate crowding kernel. Accordingly, we evaluate three task families that differ in output structure, data source, and relevant notions of crowding: short stories, alternative uses of common objects, and smartphone marketing slogans. These tasks serve three purposes: (i) to test whether the benchmark recovers idea-space narrowing relative to humans, (ii) to assess whether pairwise crowding coefficients are stable with feasible model-only sample sizes, and (iii) to examine whether human-relative crowding is sensitive to the model-conditioned generation process, using temperature and persona-mixture prompting as illustrative deployment variants.

\textbf{Tasks and human baselines.}
For short, prompted creative writing, we sample three compact-fiction prompts from the \texttt{WritingPrompts} corpus~\cite{fan2018hierarchical}: a short horror story, a 100-word parachute-failure story, and a microfiction prompt contrasting a character's life history with their final seconds. These prompts provide 87 human-contributed stories, one per author. For the alternative uses task (AUT)~\cite{guilford1967nature}, we use the \texttt{socialmuse} dataset~\cite{baten2024ai}, with 109 human contributors generating 3,047 unaided ideas across five common objects while excluding each object's primary use. For slogans, we conducted an IRB-approved study with 95 human contributors who generated 659 creative marketing slogans for a new smartphone, of which 650 are unique. Each story prompt, AUT object, or product slogan context is treated as a task condition $k$; human-relative diversity is estimated within condition before equal-weight aggregation. Dataset details and human instructions are provided in Appendix~\ref{app:human_baselines}.

\textbf{Models and protocols.}
Model generations are collected from GPT-5.4, Claude Sonnet 4.5, and Gemini 2.5 Flash under matched task instructions. The main benchmark uses neutral prompting at temperature $T=1.0$ with 50 model-only generations per task condition. Deployment variants include temperature sweeps and a persona-mixture protocol crossing a $2^5$ grid of Big Five binary personality dimensions~\cite{soto2017short}. The persona mixture induces heterogeneous generation contexts while holding the task fixed. Exact prompts and implementation details are provided in Appendix~\ref{app:model_prompts}.

\textbf{Crowding kernels.}
Our primary kernel is semantic crowding. Let $f(x)$ be a normalized sentence embedding of output $x$. We define
\begin{equation}
K^{\mathrm{sem}}_k(x,y)
=
\frac{1+\cos(f(x),f(y))}{2}.
\label{eq:semantic_kernel}
\end{equation}
This maps embedding similarity to $[0,1]$ and is applied across all task families. Because crowding can occur at different representational levels, we also evaluate task-specific kernels: plot-synopsis crowding for narrative convergence in stories, concept-bucket crowding for repeated underlying uses in AUT responses, and lexical-template overlap for repeated wording or phrase structure in slogans. These kernels test whether the benchmark conclusions persist under domain-sensitive definitions of idea-space overlap; implementation details are provided in Appendix~\ref{app:kernels}.

\textbf{Matched estimation.}
For each model $m$, task condition $k$, and kernel $K_k$, we estimate human and model-only crowding using matched-sample bootstrapping. Let $n^H_k$ be the number of human sampling units (authors for stories, participants for AUT, and slogans) and $n^A_{m,k}$ the number of model-only generations. Each bootstrap replicate samples $b_{m,k}=\min\{n^H_k,n^A_{m,k}\}$ human units and model generations with replacement, then computes the mean off-diagonal pairwise crowding:
\begin{equation}
\widehat{\kappa}^{H}_{k}
=
\frac{1}{b_{m,k}(b_{m,k}-1)}
\sum_{i\ne j} K_k(h_i,h_j),
\qquad
\widehat{\kappa}^{A}_{m,k}
=
\frac{1}{b_{m,k}(b_{m,k}-1)}
\sum_{i\ne j} K_k(a_i,a_j).
\label{eq:sample_kappas}
\end{equation}
Within each task condition, we compute $\widehat{\Delta}_{m,k}$ and $\widehat{\rho}_{m,k}$ using Eqs.~\ref{eq:delta} and~\ref{eq:rho}; task-family estimates are then obtained by equal-weight aggregation across conditions. Confidence intervals are percentile bootstrap intervals. For AUT and slogans, where human participants may contribute multiple responses, we use participant-aware sampling: sample participants first, then sample one response per selected participant. This prevents high-fluency participants from dominating the human baseline.

\textbf{Aggregation and diagnostics.}
All quantities are estimated within a task-condition and then averaged equally across conditions within a task family, so no prompt dominates because it has more human responses. We assess finite-sample stability using rarefaction curves over model-only sample size.

\begin{figure}[t]
\centering
\includegraphics[width=\linewidth]{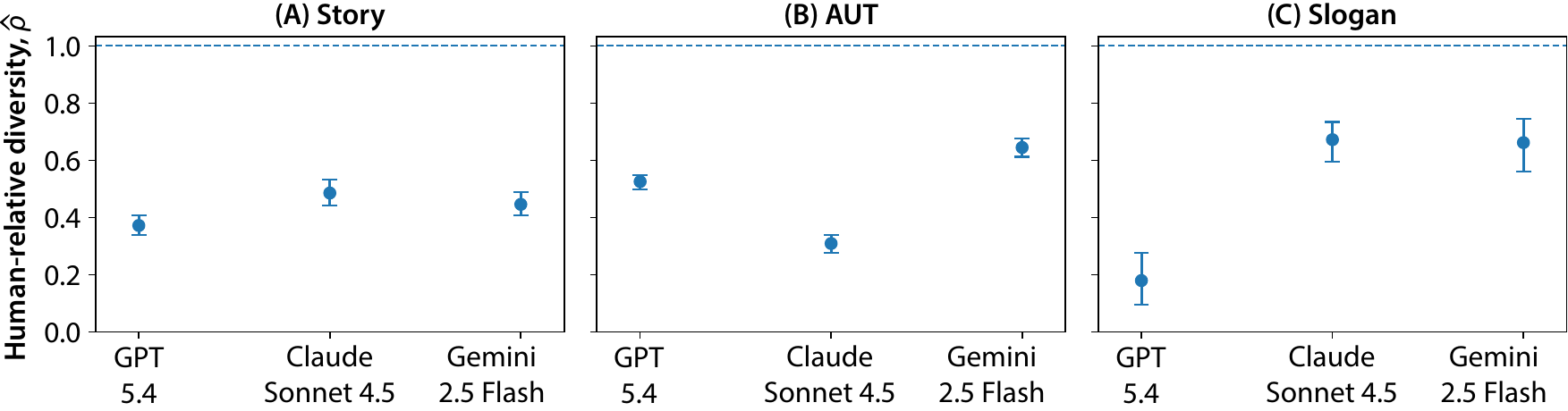}
\caption{
\textbf{Human-relative diversity under the primary semantic kernel.}
Points show task-family estimates of $\widehat{\rho}$ for each model; bars show 95\% bootstrap intervals. The dashed line marks $\rho=1$, the no-excess-crowding condition from Proposition~\ref{prop:parity}. Full numeric estimates are reported in Appendix~\ref{app:main_semantic_results}.
}
\label{fig:main_rho}
\end{figure}
\section{Results}
\label{sec:results}
\subsection{Neutral model-conditions fall below human-relative diversity parity}
\label{subsec:main_results}
Figure~\ref{fig:main_rho} reports the main semantic benchmark under neutral prompting and the primary semantic kernel. The dashed line marks $\rho=1$, the no-excess-crowding boundary from Proposition~\ref{prop:parity}. All nine neutral model-task combinations fall below this boundary, and in every case, the upper 95\% bootstrap confidence bound is also below one. Thus, under the primary semantic kernel, all evaluated neutral model-conditioned source distributions introduce positive excess crowding relative to the matched human baseline.

The estimates are interpretable as human-relative deficits rather than model-only diversity scores. Their magnitudes vary across task and model conditions, consistent with the task-relative design of the benchmark. For example, the slogan condition shows a large deficit for GPT-5.4, with $\widehat{\rho}=0.179$ and $\widehat{\Delta}=0.331$, while AUT shows a large deficit for Claude Sonnet 4.5, with $\widehat{\rho}=0.309$ and $\widehat{\Delta}=0.275$. These examples illustrate that excess crowding is not a single global property of a base model; it is estimated for a specific model-condition, task condition, and crowding kernel. Full numeric estimates are reported in Appendix~\ref{app:main_semantic_results}.

\subsection{Crowding estimates stabilize with feasible model-only sample sizes}
\label{subsec:stability_results}
A development-time benchmark is useful only if its estimates stabilize with feasible model-only sample sizes. We therefore use rarefaction to assess whether the pairwise crowding estimates underlying $\widehat{\Delta}_{m,k}$ are sensitive to the number of model generations sampled. This diagnostic concerns estimation rather than deployment scale: $n$ is the number of model-only samples used to estimate the source-level crowding coefficient, whereas $X_{-i}$ in the adoption game is the number of other creators using the same model-condition. Stability in $n$ therefore supports reliable estimation of $\widehat{\Delta}_{m,k}$; it does not imply that population-level redundancy stops growing with adoption.

\begin{figure}[t]
\centering
\includegraphics[width=\linewidth]{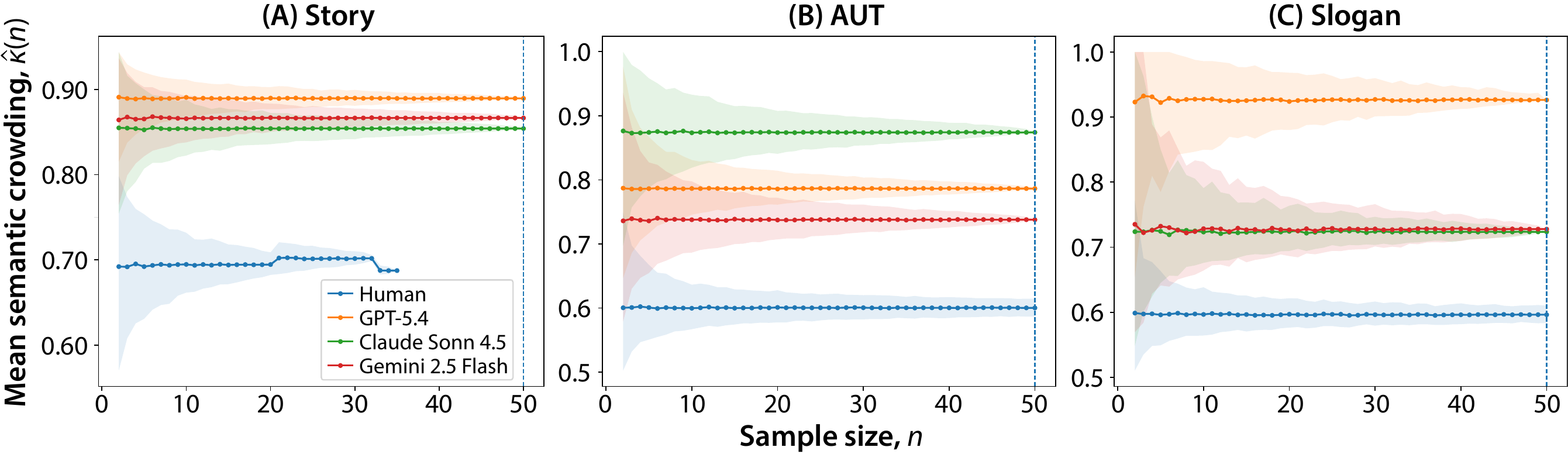}
\caption{
\textbf{Finite-sample stability of semantic crowding estimates.}
Curves show $\widehat{\kappa}(n)$ as a function of sampled responses $n$, averaged across conditions within each task family. Shaded bands show 95\% intervals from repeated rarefaction samples.
}
\label{fig:rarefaction}
\end{figure}

Figure~\ref{fig:rarefaction} shows that semantic crowding estimates stabilize within the sampled range. Between $n=40$ and $n=50$, the largest relative change in task-level $\widehat{\kappa}^A$ is only $0.104\%$, observed for Gemini 2.5 Flash on slogans; all other model-task combinations change by less than $0.06\%$ (see Appendix~\ref{app:rarefaction}). The narrow intervals near $n=50$ indicate that the chosen model-only sample size is sufficient for stable estimation in this benchmark. This supports the practicality of \textit{ex ante} evaluation: pairwise crowding tendencies can be estimated with modest model-only samples, while the adoption game separately describes how those tendencies compound with population exposure.

\subsection{Excess crowding implies adoption-dependent benefit thresholds}
\label{subsec:critical_benefit}
Proposition~\ref{prop:critical_benefit} gives the benchmarked coefficient a decision-theoretic interpretation. The empirical $\widehat{\Delta}_{m,k}$ estimates do not measure user preferences directly; rather, they imply how large the private AI advantage would need to be to offset redundancy costs at a given exposure level. Normalizing by the task's distinctiveness value gives
\begin{equation}
\frac{B^{\mathrm{crit}}_{m,k}(X)}{\gamma_k}
=
1-\exp\{-X\widehat{\Delta}_{m,k}\}.
\label{eq:empirical_bcrit}
\end{equation}
This quantity is the fraction of the task's distinctiveness value that AI use must compensate for when $X$ other creators use the same model-condition.

Even moderate pairwise excess crowding becomes strategically meaningful as exposure grows. With only one other adopter, the implied critical benefit ranges from $12.4\%$ to $28.2\%$ of the full distinctiveness value $\gamma_k$. At $X=10$, the threshold ranges from $73.3\%$ to $96.4\%$ of $\gamma_k$; by $X=25$, every model-task curve exceeds $96\%$ (see Appendix~\ref{app:critical_benefit}). In the adoption game, this means that once 25 other creators draw from the same below-parity source, the private benefit of AI must offset nearly the entire distinctiveness value in each evaluated model-task setting. This is the empirical counterpart of Proposition~\ref{prop:mass_limit}: below-parity models may be rational in private or low-exposure settings, but widespread adoption pushes users toward the full distinctiveness penalty.

\subsection{Robustness across narrative, lexical, and concept-level kernels}
\label{subsec:kernel_robustness}
We next test whether the below-parity conclusion persists under task-specific crowding kernels. For stories, all models remain below parity under a plot-synopsis kernel: GPT-5.4 moves from $\widehat{\rho}=0.372$ to $0.509$, Claude Sonnet 4.5 from $0.485$ to $0.594$, and Gemini 2.5 Flash from $0.446$ to $0.519$. Thus, story crowding persists when similarity is computed over plot synopses rather than full prose.

For slogans, all three models remain below parity under both lexical-template kernels. GPT-5.4 shows substantial lexical crowding, with $\widehat{\rho}=0.368$ under non-stopword Jaccard and $\widehat{\rho}=0.305$ under character-trigram Jaccard. Claude Sonnet 4.5 and Gemini 2.5 Flash are less lexically crowded, but their lexical $\widehat{\rho}$ values remain below one. Thus, slogan crowding persists when overlap is defined by reused words and phrase templates rather than semantic embedding similarity alone (see Appendix~\ref{app:kernel_robustness}).

For AUT, crowding also persists under concept-bucket co-membership. All three models remain below parity: GPT-5.4 has $\widehat{\rho}=0.866$, Claude Sonnet 4.5 has $\widehat{\rho}=0.715$, and Gemini 2.5 Flash has $\widehat{\rho}=0.938$. Under this kernel, $\widehat{\kappa}^A$ is the probability that two model-only generations express the same underlying use concept, so below-parity $\widehat{\rho}$ indicates that model generations repeat use concepts more often than the matched human baseline. Together, these results illustrate the task-relative evaluation logic: developers should specify a crowding kernel that matches the form of overlap relevant to their creative domain. In our benchmark, human-relative excess crowding persists under narrative-level overlap for stories, lexical-template overlap for slogans, and concept-level overlap for AUT, showing how the same framework can be instantiated with domain-appropriate kernels.

\subsection{Generation protocols can improve human-relative diversity}
\label{subsec:protocol_improvement}
Finally, we test whether human-relative crowding changes under generation-protocol variants. We use two illustrative levers---persona-mixture prompting and temperature tuning---to vary the model-conditioned source distribution while holding the task fixed.

\begin{figure}[t]
\centering
\includegraphics[width=0.85\linewidth]{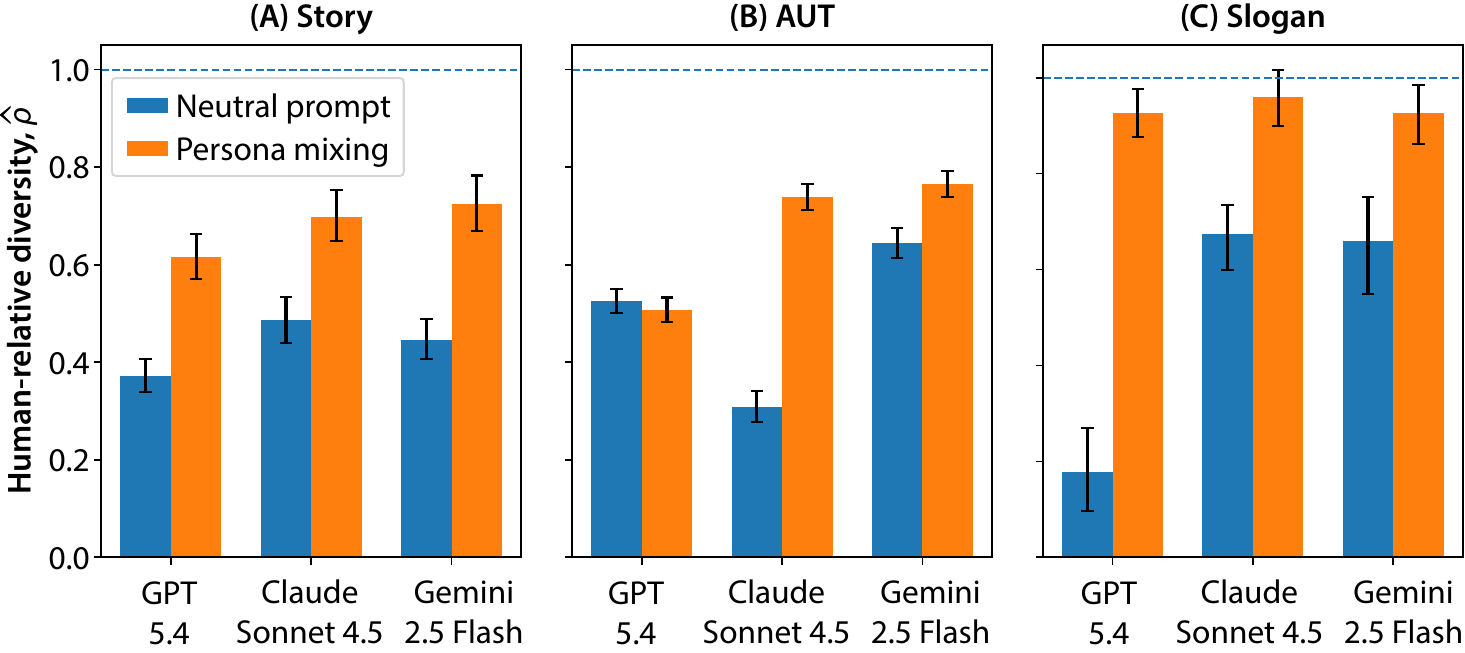}
\caption{
\textbf{Persona-mixture prompting improves human-relative diversity.}
Bars compare the neutral main protocol at $T=1.0$ with a persona-mixture protocol at $T=1.0$. Error bars show 95\% bootstrap intervals. The dashed line marks $\rho=1$, the no-excess-crowding condition.
}
\label{fig:persona_mix_actionability}
\end{figure}

Figure~\ref{fig:persona_mix_actionability} shows that persona-mixture prompting increases $\widehat{\rho}$ in eight of nine model-task combinations. The only exception is GPT-5.4 on AUT, where the confidence interval includes zero (95\% CI $[-0.053,0.017]$). The largest improvement occurs for GPT-5.4 on slogans: $\widehat{\rho}$ increases by $0.749$ and $\widehat{\Delta}$ decreases by $0.301$. Claude Sonnet 4.5 nearly reaches parity on slogans, moving from $\widehat{\rho}=0.673$ to $0.960$, with a confidence interval crossing one. Thus, crowding is not fixed by the base model alone; it can change substantially with the generation protocol. Notably, by Proposition~\ref{prop:critical_benefit}, reducing $\widehat{\Delta}$ also lowers the implied critical-benefit threshold. At $X=10$, persona-mixture prompting lowers the normalized critical-benefit threshold in eight of nine model-task settings, with the largest reductions for slogan generation (see Appendix~\ref{app:protocol_diagnostics}).

Temperature tuning provides a second protocol variant. Across the tested temperature grids, $\widehat{\rho}$ increases and $\widehat{\Delta}$ decreases from the lowest to the highest temperature in all nine model-task combinations, with strict monotonic ordering in six (see Appendix~\ref{app:protocol_diagnostics}). Together, these results show that human-relative crowding is an estimable property of a model-condition, not only a descriptive property of a base model. This makes $\widehat{\rho}$ and $\widehat{\Delta}$ actionable evaluation targets: they can be used to test whether generation protocols reduce population-level crowding before deployment.

\section{Discussion}
\label{sec:discussion}
This work reframes AI-induced diversity collapse as a development-time evaluation problem. Prior human--AI studies show that AI assistance can improve individual creative outputs while making the population of outputs more homogeneous~\cite{anderson2024homogenization,doshi2024generative,kumar2025human}. Related model-side work shows that LLMs themselves can produce unusually homogeneous creative responses, suggesting that source-level similarity is a plausible mechanism behind downstream human--AI convergence~\cite{wenger2025we,xu2025echoes}. Our framework formalizes this source-to-population link. If many users draw inspiration from the same model-conditioned source distribution, then excess concentration in that source distribution becomes a population-level crowding risk. Comparing model-only generations to matched unaided human baselines therefore yields an excess-crowding coefficient and a human-relative parity condition before new human--AI data are collected. This does not displace human--AI experiments; it gives developers a principled way to evaluate model-conditions, prompts, decoding settings, and other interventions before selecting which ones warrant costly human-subject testing.


Our empirical results suggest that this evaluation logic remains tractable across disparate creative tasks. Across three task families, we find that neutral model conditions consistently fall below human-relative parity, that crowding estimates stabilize with feasible sample sizes, and that protocol variants (such as temperature and persona-mixture prompting) can significantly shift $\widehat{\rho}$ and $\widehat{\Delta}$. We view these variants as proof-of-concept explorations rather than final prescriptions. The broader implication is that human-relative crowding is a measurable, estimable property of a model condition, rendering diversity collapse an actionable target for optimization prior to deployment.


Crucially, the framework is designed as a generalizable procedure rather than a rigid task suite. The three task families analyzed here were selected to capture varied output lengths, response structures, and dimensions of crowding: compact stories for narrative convergence, slogans for highly constrained creative spaces, and the AUT for classic divergent thinking. The newly collected slogan baseline is particularly instructive, providing a fresh human reference for a task unlikely to be contaminated by model training exposure. Any application of this framework requires a matched human reference and a domain-appropriate crowding kernel. This task-relative structure is a deliberate feature: it avoids evaluating model diversity in a vacuum and instead asks whether shared model use introduces excess redundancy relative to the human counterfactual for a specific creative setting.


This perspective also clarifies the relationship between population-level diversity and established metrics of individual creativity. Creativity research has long used complementary lenses: idea-intrinsic quality measures, social-rarity measures, and subjective ratings~\cite{amabile1982social,runco2012standard,beaty2021automating,bangash2025musescorer}. Our framework extends the logic of social rarity to the generative era by quantifying how shared access to a model-conditioned source reduces the distinctiveness of the population-level output. It is intended to complement, rather than replace, intrinsic quality measures. A model condition may produce high-quality outputs while suffering from severe crowding, or conversely, achieve high diversity at the expense of quality. Developers must therefore evaluate population-level crowding alongside other dimensions relevant to their creativity task.


Finally, the adoption game provides a theoretical foundation for why human-relative parity is a critical benchmark. The critical-benefit curves translate $\widehat{\Delta}$ into the marginal private benefit required to offset redundancy costs under specific exposure models. While the exponential cost function serves as a convenient monotone specification, the central insight is robust: any redundancy-cost function that increases with crowding and exposure ensures that below-parity source distributions become increasingly costly as adoption expands. At the limit of mass adoption, below-parity conditions expose users to a maximum distinctiveness penalty. Human-relative parity thus represents the fundamental condition under which shared model use introduces no excess crowding costs compared to the baseline of human convergence.

\section{Limitations}
\label{sec:limitations}
The empirical demonstration is limited to text-based creative tasks. Extending the framework to images, music, code, scientific hypotheses, and multimodal design will require domain-specific human baselines and kernels. The benchmark also estimates source-level crowding rather than the full realized interaction process: interfaces, user expertise, selective uptake, and editing behavior can mediate how model suggestions shape final human outputs.

\clearpage
\appendix

\section{Proofs of Theoretical Results}
\label{app:proofs}
\subsection{Proof of Proposition~\ref{prop:parity}}
\begin{proof}
By definition,
\[
\Delta_{m,k}=\max\{0,\kappa^A_{m,k}-\kappa^H_k\}.
\]
Therefore, $\Delta_{m,k}=0$ if and only if $\kappa^A_{m,k}\le \kappa^H_k$. Since $\kappa^H_k<1$, the denominator of $\rho_{m,k}$ is positive. Hence,
\[
\kappa^A_{m,k}\le \kappa^H_k
\quad\Longleftrightarrow\quad
1-\kappa^A_{m,k}\ge 1-\kappa^H_k
\quad\Longleftrightarrow\quad
\frac{1-\kappa^A_{m,k}}{1-\kappa^H_k}\ge 1.
\]
By Eq.~\ref{eq:rho}, this is equivalent to $\rho_{m,k}\ge1$. Thus,
\[
\Delta_{m,k}=0
\quad\Longleftrightarrow\quad
\rho_{m,k}\ge1.
\]
\end{proof}

\subsection{Proof of Proposition~\ref{prop:critical_benefit}}

\begin{proof}
A rational creator adopts AI when
\[
B_{i,m,k}>C_{m,k}(X_{-i}).
\]
Substituting the redundancy cost from Eq.~\ref{eq:redundancy_cost} gives
\[
B_{i,m,k}>
\gamma_k\left(1-\exp\{-X_{-i}\Delta_{m,k}\}\right).
\]
Therefore adoption is rational if and only if
\[
B_{i,m,k}>
B^{\mathrm{crit}}_{m,k}(X_{-i}),
\]
where
\[
B^{\mathrm{crit}}_{m,k}(X_{-i})
=
\gamma_k\left(1-\exp\{-X_{-i}\Delta_{m,k}\}\right).
\]

It remains to show the monotonicity claims. The threshold is increasing in $\gamma_k$ directly. For $X_{-i}\ge0$ and $\Delta_{m,k}\ge0$,
\[
\frac{\partial B^{\mathrm{crit}}_{m,k}}{\partial \Delta_{m,k}}
=
\gamma_k X_{-i}\exp\{-X_{-i}\Delta_{m,k}\}
\ge 0,
\]
and
\[
\frac{\partial B^{\mathrm{crit}}_{m,k}}{\partial X_{-i}}
=
\gamma_k \Delta_{m,k}\exp\{-X_{-i}\Delta_{m,k}\}
\ge 0.
\]
Thus, $B^{\mathrm{crit}}_{m,k}$ is increasing in $\Delta_{m,k}$ and $X_{-i}$.

For $\rho_{m,k}<1$, Eq.~\ref{eq:delta_from_rho} gives
\[
\Delta_{m,k}=(1-\rho_{m,k})(1-\kappa^H_k),
\]
so
\[
\frac{\partial \Delta_{m,k}}{\partial \rho_{m,k}}
=
-(1-\kappa^H_k)<0.
\]
By the chain rule,
\[
\frac{\partial B^{\mathrm{crit}}_{m,k}}{\partial \rho_{m,k}}
=
\frac{\partial B^{\mathrm{crit}}_{m,k}}{\partial \Delta_{m,k}}
\frac{\partial \Delta_{m,k}}{\partial \rho_{m,k}}
\le 0.
\]
Hence, for $\rho_{m,k}<1$, the critical benefit threshold is decreasing in $\rho_{m,k}$.
\end{proof}

\subsection{Proof of Eq.~\ref{eq:expected_cost}}
\begin{proof}
Suppose each of the other $N-1$ creators independently adopts the same model-condition with probability $p$. Then
\[
X_{-i}\sim \mathrm{Binomial}(N-1,p).
\]
From Eq.~\ref{eq:redundancy_cost},
\[
C_{m,k}(X_{-i})
=
\gamma_k\left(1-\exp\{-X_{-i}\Delta_{m,k}\}\right).
\]
Taking expectations,
\[
\mathbb{E}[C_{m,k}]
=
\gamma_k\left(1-\mathbb{E}\left[\exp\{-X_{-i}\Delta_{m,k}\}\right]\right).
\]
For a binomial random variable $X\sim\mathrm{Binomial}(N-1,p)$, the probability generating function gives
\[
\mathbb{E}[t^X]=(1-p+pt)^{N-1}.
\]
Setting $t=\exp\{-\Delta_{m,k}\}$ yields
\[
\mathbb{E}\left[\exp\{-X_{-i}\Delta_{m,k}\}\right]
=
\left(1-p+p\exp\{-\Delta_{m,k}\}\right)^{N-1}.
\]
Therefore,
\[
\mathbb{E}[C_{m,k}]
=
\gamma_k
\left[
1-
\left(
1-p+p\exp\{-\Delta_{m,k}\}
\right)^{N-1}
\right].
\]
\end{proof}

\subsection{Proof of Proposition~\ref{prop:mass_limit}}
\begin{proof}
If $\rho_{m,k}\ge1$, then Proposition~\ref{prop:parity} implies $\Delta_{m,k}=0$. Substituting into Eq.~\ref{eq:redundancy_cost},
\[
C_{m,k}(X_{-i})
=
\gamma_k(1-\exp\{0\})
=
0
\]
for all $X_{-i}$.

If $\rho_{m,k}<1$, then Eq.~\ref{eq:delta_from_rho} implies
\[
\Delta_{m,k}=(1-\rho_{m,k})(1-\kappa^H_k)>0.
\]
Therefore,
\[
\exp\{-X_{-i}\Delta_{m,k}\}\to0
\quad\text{as}\quad
X_{-i}\to\infty.
\]
Substituting into Eq.~\ref{eq:redundancy_cost},
\[
\lim_{X_{-i}\to\infty} C_{m,k}(X_{-i})
=
\gamma_k(1-0)
=
\gamma_k.
\]
Thus,
\[
\lim_{X_{-i}\to\infty}C_{m,k}(X_{-i})
=
\begin{cases}
0, & \rho_{m,k}\ge1,\\
\gamma_k, & \rho_{m,k}<1.
\end{cases}
\]
\end{proof}

\section{Human Baseline Datasets and Task Conditions}
\label{app:human_baselines}
This appendix describes the human baselines used to instantiate the task-relative benchmark. All duplicate human outputs are retained because duplicate or near-duplicate responses are part of the crowding signal.

\subsection{Short, prompted creative writing}
\label{app:story_human_baseline}
The story baseline is drawn from the \texttt{WritingPrompts} corpus~\citep{fan2018hierarchical}, a Reddit-based dataset of human-written stories paired with writing prompts. The original corpus contains 303,358 prompt--story pairs, split into 272,600 training, 15,620 validation, and 15,138 test stories. The corpus was collected from Reddit's \texttt{r/WritingPrompts}, where users submit story premises and other users write responses; each prompt can therefore have multiple human-written stories. The dataset is distributed through the \texttt{fairseq} story-generation example page; we report the prompt IDs and processing details needed to reproduce our filtering from the original release.

We selected three high-response prompts that asked participants to produce compact narrative fiction under shared constraints (see Table~\ref{tab:story_conditions}). Compact prompts were useful for the benchmark because they reduce uncontrolled variation in response length while preserving narrative choice. Each prompt is treated as a separate condition $k$.

\begin{table}[h]
\centering
\caption{
\textbf{\texttt{WritingPrompts} story conditions.}
Here, $n^H_k$ denotes the number of available human sampling units for condition $k$. Because each retained \texttt{WritingPrompts} response is one story from one human author, $n^H_k$ equals the number of human stories for each prompt. Word and sentence counts are computed over these retained human stories.
}
\label{tab:story_conditions}
\small
\begin{tabularx}{\linewidth}{lXrrr}
\toprule
Prompt ID & Prompt & $n^H_k$ & Median words & Median sentences \\
\midrule
10491 &
A short horror story. Something to chill the bones in one hundred words or less. &
35 & 104 & 10.0 \\
93742 &
100 Words or Less --- The parachute isn't opening up &
32 & 107 & 14.0 \\
93855 &
Describe 100 years of a character's life in 10 words. Then describe the last 10 seconds of their life in 100 words. &
20 & 112 & 10.5 \\
\bottomrule
\end{tabularx}
\end{table}

\subsection{Alternative Uses Task}
\label{app:aut_human_baseline}
The Alternative Uses Task (AUT) baseline comes from the \texttt{socialmuse} dataset~\citep{baten2024ai}. Participants completed a divergent thinking task in which they generated alternative uses for common objects while excluding the object's primary use. The retained human baseline contains 109 participants, 5 object conditions, and 3,047 unaided ideas (see Table~\ref{tab:aut_conditions}). Participants contributed up to six ideas per object; the median participant produced 30 ideas across the five objects. Responses range from short phrases to full sentences. For example, for the object ``shoe,'' human responses include ``We can use a shoe as a hamster bed'' and ``As a doorstop.'' In participant-aware bootstrapping, participants are sampled first, and then one response is sampled per selected participant within each object condition. The original authors provided the dataset and granted permission for research use.

\begin{table}[h]
\centering
\caption{
\textbf{Alternative Uses Task conditions.}
Only unaided human ideas are retained. Here, $n^H_k$ denotes the number of available human sampling units for condition $k$; because participant-aware sampling is used, $n^H_k$ equals the number of participants.
}
\label{tab:aut_conditions}
\small
\begin{tabularx}{\linewidth}{lXrrrr}
\toprule
Object & Common use & $n^H_k$ & Ideas & Unique ideas & Median words \\
\midrule
Shoe & used as footwear & 109 & 604 & 532 & 3 \\
Button & used to fasten things & 109 & 603 & 583 & 4 \\
Key & used to open a lock & 109 & 612 & 574 & 4 \\
Wooden pencil & used for writing & 109 & 613 & 583 & 4 \\
Automobile tire & used on the wheel of an automobile & 109 & 615 & 575 & 5 \\
\bottomrule
\end{tabularx}
\end{table}

\subsection{Smartphone marketing slogans}
\label{app:slogan_human_baseline}
We collected unaided slogan baseline data through an IRB-approved online study with U.S.-based participants recruited from Prolific (see Table~\ref{tab:slogan_demographics}). All participants provided informed consent before beginning the study and were paid at the local minimum hourly rate at the end of the task. Participants were asked to generate creative marketing slogans for a hypothetical new smartphone. The retained dataset contains 659 slogans from 95 participants, of which 650 are unique strings (see Table~\ref{tab:slogan_condition}). The median slogan length is 4 words; only one slogan exceeds the requested six-word maximum. Exact duplicate slogans are retained because repeated slogans are part of the human crowding baseline. Participants saw the following instruction:

\begin{quote}
You are part of the marketing team at a tech company preparing to launch a new smartphone. You are tasked with coming up with creative new marketing slogans for this brand-new smartphone. Your task is to:
\begin{itemize}
    \item Come up with as many creative marketing slogans as you can for this brand new smartphone. The slogans need to be novel and appropriate.
    \item You must create at least 3 slogans, but please try to generate as many slogans as possible.
    \item Each slogan should not exceed 6 words.
    \item You can assume any detail about the smartphone -- be as creative as you want!
    \item You will have 3 minutes for this task.
\end{itemize}
\end{quote}

\begin{table}[h]
\centering
\caption{
\textbf{Demographic composition of the smartphone slogan study.}
Counts are reported for the participants who completed the slogan task and demographic questionnaire. Race categories are not mutually exclusive if participants selected multiple categories.
}
\label{tab:slogan_demographics}
\small
\begin{tabular}{llr}
\toprule
Category & Group & Number of participants \\
\midrule
Age & 18--24 & 12 \\
Age & 25--34 & 31 \\
Age & 35--44 & 30 \\
Age & 45--54 & 10 \\
Age & 55--64 & 9 \\
Age & 65+ & 3 \\
\midrule
Gender & Male & 46 \\
Gender & Female & 48 \\
Gender & Non-binary & 1 \\
\midrule
Education & High school graduate & 12 \\
Education & Some college & 16 \\
Education & Associate degree & 6 \\
Education & Bachelor's degree & 39 \\
Education & Master's degree & 17 \\
Education & Professional degree & 3 \\
Education & Doctorate degree & 2 \\
\midrule
Hispanic Origin & Yes & 14 \\
Hispanic Origin & No & 81 \\
\midrule
Race & American Indian or Alaska Native & 1 \\
Race & Asian & 20 \\
Race & Black or African American & 14 \\
Race & Native Hawaiian or Other Pacific Islander & 1 \\
Race & Other & 5 \\
Race & White & 58 \\
\bottomrule
\end{tabular}
\end{table}

\begin{table}[h]
\centering
\caption{
\textbf{Smartphone slogan condition.} $n^H_k$ denotes the number of available human sampling units for condition $k$; because participant-aware sampling is used, $n^H_k$ equals the number of participants.
}
\label{tab:slogan_condition}
\small
\begin{tabular}{lrrrrrr}
\toprule
Condition & Slogans & Unique slogans & $n^H_k$ & Mean slogans / participant  & Median words \\
\midrule
Smartphone & 659 & 650 & 95 & 6.94 & 4 \\
\bottomrule
\end{tabular}
\end{table}

As with the AUT baseline, participant-aware bootstrapping samples participants first and then samples one slogan per selected participant. This prevents high-fluency participants from dominating the human baseline.

\section{Model Prompts and Generation Protocols}
\label{app:model_prompts}
This appendix documents the model-only generation protocols used in the benchmark. For every task condition, models were prompted to generate exactly one creative product per request. This makes the model sampling unit parallel to the human sampling unit used in the matched benchmark.

\subsection{Models and generation scenarios}
We generated outputs from GPT-5.4, Claude Sonnet 4.5, and Gemini 2.5 Flash. For each model-task pair, the main benchmark used neutral prompting at temperature $T=1.0$ with 50 model-only generations per task condition. We also collected deployment-variant samples using neutral temperature changes and persona-mixture prompting (Table~\ref{tab:model_generation_scenarios}).

\begin{table}[h]
\centering
\caption{
\textbf{Model generation scenarios.}
The main benchmark uses neutral prompting at $T=1.0$. Temperature and persona-mixture generations are used for deployment-variant analyses.
}
\label{tab:model_generation_scenarios}
\small
\begin{tabularx}{\linewidth}{lXrr}
\toprule
Scenario & Description & Generations / condition & Personas \\
\midrule
Neutral main & Neutral prompt at $T=1.0$ & 50 & 1 \\
Temperature robustness & Neutral prompt at non-default temperatures & 10 per temperature & 1 \\
Persona mixture & 32 persona profiles crossed with temperature & 10 per persona-temperature pair & 32 \\
\bottomrule
\end{tabularx}
\end{table}

For the neutral temperature-robustness scenario, the non-default temperatures were $\{0.7,1.3\}$ for GPT-5.4 and Gemini 2.5 Flash, and $\{0.3,0.7\}$ for Claude Sonnet 4.5. The $T=1.0$ neutral condition is supplied by the main benchmark.

\subsection{Neutral system instructions}
For stories, the neutral system instruction was:

\begin{quote}
You are participating in a creative writing task. Respond to the prompt as a human participant would. Follow the prompt's constraints exactly. Do not explain your answer. Do not include commentary before or after the creative response.
\end{quote}

For AUT and slogans, the neutral system instruction was:

\begin{quote}
You are participating in a creativity test. Follow the task instructions exactly. Do not explain your reasoning. Do not include commentary before or after the answer.
\end{quote}

\paragraph{Task-specific user prompts.}

\textit{Stories.}
For each selected \texttt{WritingPrompts} condition, the user prompt was:

\begin{quote}
Prompt: \\
\texttt{<story prompt>}
\end{quote}

where \texttt{<story prompt>} was replaced with the corresponding prompt text in Table~\ref{tab:story_conditions}.

\textit{Alternative Uses Task.}
For each AUT object, the user prompt was:

\begin{quote}
Object: \texttt{<object>} \\
Common use to avoid: \texttt{<common use>} \\

Generate exactly one unusual, creative, and plausible alternative use for the object or one of its parts. Do not use the common use. Do not list multiple uses. Return only the alternative use as a short phrase or one sentence.
\end{quote}

The object and common-use pairs are listed in Table~\ref{tab:aut_conditions}.

\textit{Smartphone slogans.}
For the slogan task, the user prompt was:

\begin{quote}
You are part of the marketing team at a tech company preparing to launch a new smartphone.

Generate exactly one creative marketing slogan for this brand new smartphone.

Requirements:
\begin{itemize}
    \item The slogan must be novel and appropriate.
    \item The slogan must not exceed 6 words.
    \item You may assume any detail about the smartphone.
    \item Do not list multiple slogans.
    \item Return only the slogan text.
\end{itemize}
\end{quote}

\subsection{Persona-mixture prompting}
The persona-mixture protocol crossed five binary personality dimensions, yielding $2^5=32$ profiles: 
\[
\begin{aligned}
&\{\text{extroverted},\ \text{introverted}\}
\times \{\text{agreeable},\ \text{antagonistic}\} \\
&\quad \times \{\text{conscientious},\ \text{unconscientious}\}
\times \{\text{neurotic},\ \text{emotionally stable}\} \\
&\quad \times \{\text{open to experience},\ \text{closed to experience}\}.
\end{aligned}
\]

For stories, the following instruction was appended to the neutral system instruction:

\begin{quote}
Write as if you are a person with the following personality profile: \texttt{<traits>}. Let this personality influence the creative choices, tone, voice, imagery, pacing, and emotional framing of the story, while still following the user's writing prompt exactly.
\end{quote}

For AUT, the following instruction was appended:

\begin{quote}
Write as if you are a person with the following personality profile: \texttt{<traits>}. Let this personality influence the creative choice, style, and framing of the idea, while still following the task exactly.
\end{quote}

For slogans, the following instruction was appended:

\begin{quote}
Write as if you are a person with the following personality profile: \texttt{<traits>}. Let this personality influence the creative choice, style, tone, and framing of the slogan, while still following the task exactly.
\end{quote}

In each case, \texttt{<traits>} was replaced by the comma-separated traits in one of the 32 personality profiles.

\subsection{Output limits}
Maximum output lengths were set by task family: 800 tokens for stories, 120 tokens for AUT responses, and 40 tokens for slogans. These limits were ceilings rather than target lengths; all prompts instructed models to return only the requested creative product.

\section{Crowding Kernels}
\label{app:kernels}
The benchmark is task-relative: each kernel defines what it means for two outputs to occupy nearby regions of idea-space in a given task. The primary semantic kernel in Eq.~\ref{eq:semantic_kernel} provides a common measure across all task families. For all primary semantic analyses, we use \texttt{sentence-transformers/all-mpnet-base-v2} to compute sentence embeddings and normalize embeddings before computing cosine similarity. The same embedding model is also used for the plot-synopsis story kernel and for nearest-neighbor retrieval in the AUT concept-bucketing pipeline.

\subsection{Story plot-synopsis kernel}
\label{app:story_plot_kernel}
Full-story semantic similarity can reflect plot, style, setting, tone, and surface phrasing. To isolate narrative-content crowding, we generate a one-sentence plot synopsis $s(x)$ for each story and compute semantic crowding over synopsis embeddings. Synopses were generated with GPT-5.4 at temperature $T=0$ and a maximum output length of 80 tokens. Duplicate story texts were summarized once and then mapped back to all corresponding story rows, preserving duplicates for the crowding analysis.

The system instruction for synopsis generation was:

\begin{quote}
You convert creative stories into concise plot synopses for research analysis. Follow the instructions exactly. Do not add commentary.
\end{quote}

For each story, the user prompt was:

\begin{quote}
Summarize the following story as exactly one sentence.

Focus on the plot: the protagonist or central entity, the main situation, the central conflict or change, and the outcome if it is present.

Do not evaluate the writing quality. Do not mention the author. Do not quote the story. Do not include labels such as `Synopsis:'. Return only the one-sentence plot synopsis.

STORY:

\texttt{<story text>}
\end{quote}

We manually checked a random $10\%$ of all story synopses generated by this method, and found all of those synopses to be appropriate. 

Let $s(x)$ denote the generated synopsis for story $x$. We embed synopses using \texttt{sentence-transformers/all-mpnet-base-v2} and compute
\begin{equation}
K^{\mathrm{plot}}_k(x,y)
=
\frac{1+\cos(f(s(x)),f(s(y)))}{2}.
\label{eq:plot_kernel}
\end{equation}
This kernel measures whether stories converge on similar plot content rather than only on similar full-prose representations.  

\subsection{AUT concept-bucket kernel}
\label{app:aut_bucket_kernel_definition}
For alternative uses, semantic similarity may obscure whether two responses express the same underlying use concept. We therefore construct object-specific concept codebooks using a retrieval-assisted LLM bucketing pipeline psychometrically validated in \texttt{MuseScorer}~\cite{bangash2025musescorer}. Human and model responses are pooled within each object condition, and each object receives its own codebook because bucket IDs are meaningful only within object condition.

For each AUT object, the pooled file includes human responses and model responses from the main and deployment-variant generation scenarios. Exact duplicate responses are retained because repeated responses are part of the crowding signal. Each pooled response is assigned to a concept bucket using the original annotation utility with the following configuration: LLM judge \texttt{llama3.3:70b-Instruct} through Ollama, embedding model \texttt{sentence-transformers/all-mpnet-base-v2}, $K_c=10$ nearest comparison candidates, and \texttt{CoT} prompt method, based on the exact LLM settings validated to be the best-performing by the original authors. This bucketing was conducted using an Intel Core i7-based computer with $64$GB RAM and an RTX $3070$ Ti graphics card. The final result generation took roughly $5$ GPU days.

The bucketing procedure proceeds sequentially within each object. For a new response, the pipeline retrieves the $K_c=10$ nearest existing bucket exemplars under the embedding model, then asks the LLM judge whether the response is an obvious rephrasing of an existing concept or should form a new bucket. The common use of each object is supplied as a forbidden idea and inserted as comparison idea ID 0 in the original pipeline.

Let $b_k(x)$ denote the assigned bucket for response $x$ within object condition $k$. The concept-bucket kernel is
\begin{equation}
K^{\mathrm{bucket}}_k(x,y)
=
\mathbf{1}\{b_k(x)=b_k(y)\}.
\label{eq:bucket_kernel}
\end{equation}
Under this kernel, $\kappa^H_k$ is the probability that two human responses express the same underlying use concept, and $\kappa^A_{m,k}$ is the corresponding probability for model-only generations.

\subsection{Slogan lexical-template kernels}
\label{app:slogan_lexical_kernel_definition}
For slogans, short length makes lexical reuse especially important. We therefore compute two lexical-template kernels in addition to the primary semantic kernel. Before computing lexical overlap, slogans are lowercased and normalized by stripping punctuation and collapsing whitespace.

The first lexical kernel is non-stopword token Jaccard similarity:
\begin{equation}
K^{\mathrm{word}}_k(x,y)
=
\frac{|T(x)\cap T(y)|}{|T(x)\cup T(y)|},
\label{eq:word_jaccard_kernel}
\end{equation}
where $T(x)$ is the set of non-stopword tokens in slogan $x$. The second is character-trigram Jaccard similarity:
\begin{equation}
K^{\mathrm{tri}}_k(x,y)
=
\frac{|G_3(x)\cap G_3(y)|}{|G_3(x)\cup G_3(y)|},
\label{eq:trigram_kernel}
\end{equation}
where $G_3(x)$ is the set of character trigrams after lowercasing and normalizing punctuation and whitespace. The word-level kernel captures repeated content words, while the trigram kernel captures more fine-grained phrase-template overlap.

\section{Additional Main Semantic Crowding Results}
\label{app:main_semantic_results}
\subsection{Task-family semantic benchmark}
\label{app:main_semantic_table}
Table~\ref{tab:main_results_app} reports the full numeric estimates for the main semantic benchmark shown visually in Fig.~\ref{fig:main_rho}.

\begin{table}[h]
\centering
\caption{
\textbf{Semantic crowding benchmark.}
Each row reports equal-weight task-family estimates under the primary semantic kernel. Confidence intervals are percentile bootstrap intervals for $\widehat{\rho}$.
}
\label{tab:main_results_app}
\small
\begin{tabular}{llrrrrr}
\toprule
Model & Task & $\widehat{\kappa}^H$ & $\widehat{\kappa}^A$ & $\widehat{\Delta}$ & $\widehat{\rho}$ & 95\% CI for $\widehat{\rho}$ \\
\midrule
GPT-5.4 & Stories & 0.706 & 0.892 & 0.186 & 0.372 & [0.338, 0.408] \\
Claude Sonnet 4.5 & Stories & 0.706 & 0.857 & 0.151 & 0.485 & [0.441, 0.533] \\
Gemini 2.5 Flash & Stories & 0.705 & 0.869 & 0.164 & 0.446 & [0.406, 0.490] \\
GPT-5.4 & AUT & 0.601 & 0.791 & 0.190 & 0.525 & [0.499, 0.548] \\
Claude Sonnet 4.5 & AUT & 0.601 & 0.877 & 0.275 & 0.309 & [0.277, 0.340] \\
Gemini 2.5 Flash & AUT & 0.601 & 0.743 & 0.142 & 0.645 & [0.612, 0.677] \\
GPT-5.4 & Slogans & 0.597 & 0.928 & 0.331 & 0.179 & [0.096, 0.275] \\
Claude Sonnet 4.5 & Slogans & 0.597 & 0.729 & 0.132 & 0.672 & [0.596, 0.734] \\
Gemini 2.5 Flash & Slogans & 0.597 & 0.733 & 0.136 & 0.662 & [0.562, 0.746] \\
\bottomrule
\end{tabular}
\end{table}

\subsection{Human versus AI crowding}
\label{app:kappa_h_vs_a}
Figure~\ref{fig:kappa_h_vs_a} plots $\widehat{\kappa}^H$ against $\widehat{\kappa}^A$ at the task-family level. All model-task points lie above the diagonal, indicating that AI outputs are more semantically crowded than the matched human baseline in every main benchmark comparison.

\begin{figure}[h]
\centering
\includegraphics[width=0.68\linewidth]{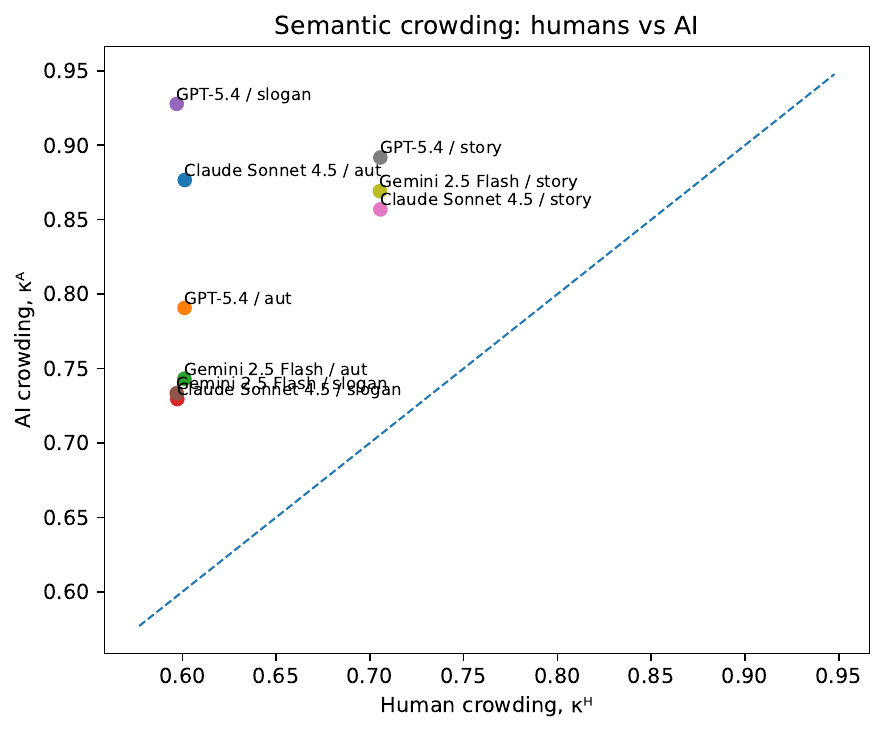}
\caption{
\textbf{Human versus AI semantic crowding.}
Each point is a model-task comparison. Points above the diagonal have $\widehat{\kappa}^A>\widehat{\kappa}^H$, corresponding to positive excess crowding.
}
\label{fig:kappa_h_vs_a}
\end{figure}

\section{Finite-Sample Stability Diagnostics}
\label{app:rarefaction}
\subsection{Task-level rarefaction drift}
\label{app:rarefaction_drift}
Table~\ref{tab:rarefaction_drift} reports the task-level change in mean semantic crowding from $n=40$ to $n=50$ for the rarefaction curves. Relative drift is computed as $|\widehat{\kappa}(50)-\widehat{\kappa}(40)|/|\widehat{\kappa}(50)|$. For human stories, the maximum available rarefaction size differs because the story prompts have fewer human responses; the reported human story drift uses $n=25$ to $n=35$.

\begin{table}[h]
\centering
\caption{
\textbf{Task-level recent drift in semantic crowding estimates.}
All values are computed from task-level rarefaction curves under the primary semantic kernel.
}
\label{tab:rarefaction_drift}
\small
\begin{tabular}{lllrrrrr}
\toprule
Source & Model & Task & $n_{\mathrm{low}}$ & $n_{\mathrm{high}}$ & $\widehat{\kappa}(n_{\mathrm{low}})$ & $\widehat{\kappa}(n_{\mathrm{high}})$ & Relative drift \\
\midrule
AI & Claude Sonnet 4.5 & AUT & 40 & 50 & 0.874064 & 0.874004 & 0.0069\% \\
AI & GPT-5.4 & AUT & 40 & 50 & 0.786562 & 0.786407 & 0.0196\% \\
AI & Gemini 2.5 Flash & AUT & 40 & 50 & 0.737719 & 0.737773 & 0.0073\% \\
Human & Human & AUT & 40 & 50 & 0.600302 & 0.600237 & 0.0109\% \\
AI & Claude Sonnet 4.5 & Slogans & 40 & 50 & 0.723630 & 0.723616 & 0.0019\% \\
AI & GPT-5.4 & Slogans & 40 & 50 & 0.926871 & 0.926327 & 0.0587\% \\
AI & Gemini 2.5 Flash & Slogans & 40 & 50 & 0.728476 & 0.727721 & 0.1038\% \\
Human & Human & Slogans & 40 & 50 & 0.596146 & 0.596149 & 0.0005\% \\
AI & Claude Sonnet 4.5 & Stories & 40 & 50 & 0.853920 & 0.854035 & 0.0135\% \\
AI & GPT-5.4 & Stories & 40 & 50 & 0.889544 & 0.889512 & 0.0036\% \\
AI & Gemini 2.5 Flash & Stories & 40 & 50 & 0.866477 & 0.866443 & 0.0039\% \\
Human & Human & Stories & 25 & 35 & 0.701182 & 0.687392 & 2.0061\% \\
\bottomrule
\end{tabular}
\end{table}

\section{Critical-Benefit Thresholds}
\label{app:critical_benefit}
Figure~\ref{fig:critical_benefit} plots the normalized critical-benefit curves implied by the main semantic excess-crowding estimates. Table~\ref{tab:critical_benefit_values} reports the same thresholds at selected exposure levels. Values are computed as $B^{\mathrm{crit}}(X)/\gamma=1-\exp(-X\widehat{\Delta})$.

\begin{figure}[h]
\centering
\includegraphics[width=\linewidth]{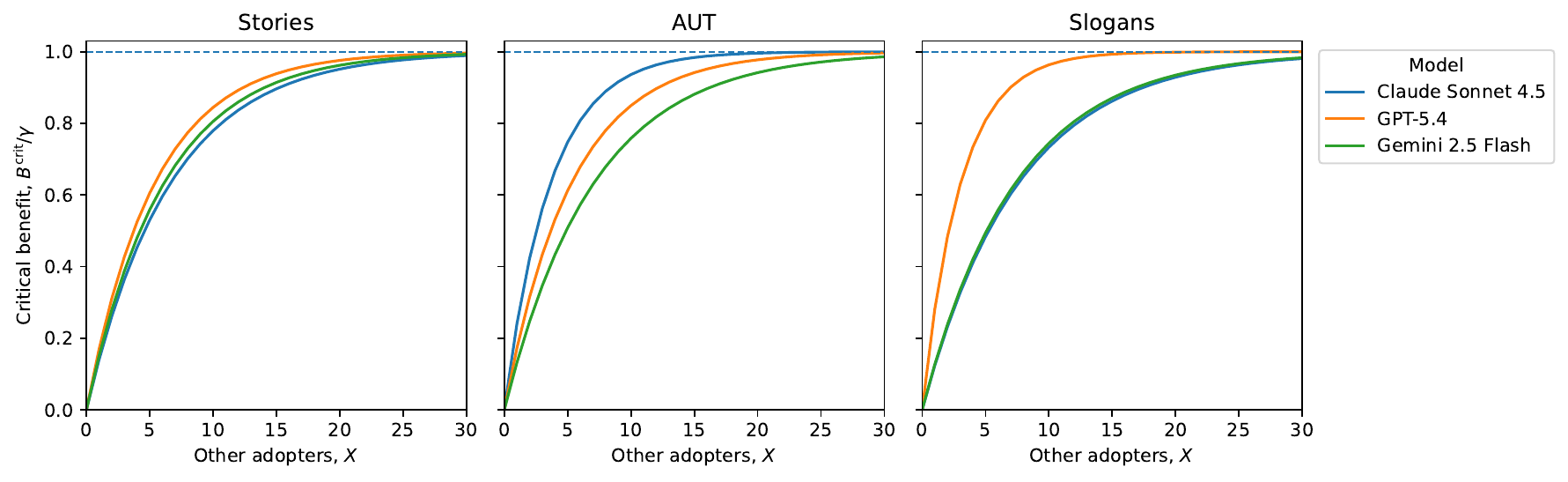}
\caption{
\textbf{Critical private benefit required for rational adoption.}
Each curve shows $B^{\mathrm{crit}}(X)/\gamma=1-\exp(-X\widehat{\Delta})$, using the main semantic excess-crowding estimate for one model-task pair. Lower human-relative diversity raises the private AI benefit required to justify adoption in crowded settings.
}
\label{fig:critical_benefit}
\end{figure}

\begin{table}[h]
\centering
\caption{
\textbf{Normalized critical benefit thresholds.}
Each entry is the fraction of the task's distinctiveness value $\gamma$ that AI must compensate for at exposure level $X$.
}
\label{tab:critical_benefit_values}
\small
\begin{tabular}{llrrrrr}
\toprule
Model & Task & $\widehat{\Delta}$ & $X=1$ & $X=5$ & $X=10$ & $X=25$ \\
\midrule
GPT-5.4 & Stories & 0.186 & 0.170 & 0.605 & 0.844 & 0.990 \\
Claude Sonnet 4.5 & Stories & 0.151 & 0.140 & 0.530 & 0.779 & 0.977 \\
Gemini 2.5 Flash & Stories & 0.164 & 0.151 & 0.560 & 0.806 & 0.983 \\
GPT-5.4 & AUT & 0.190 & 0.173 & 0.613 & 0.850 & 0.991 \\
Claude Sonnet 4.5 & AUT & 0.275 & 0.240 & 0.747 & 0.936 & 0.999 \\
Gemini 2.5 Flash & AUT & 0.142 & 0.132 & 0.508 & 0.758 & 0.971 \\
GPT-5.4 & Slogans & 0.331 & 0.282 & 0.809 & 0.964 & 1.000 \\
Claude Sonnet 4.5 & Slogans & 0.132 & 0.124 & 0.483 & 0.733 & 0.963 \\
Gemini 2.5 Flash & Slogans & 0.136 & 0.127 & 0.493 & 0.743 & 0.967 \\
\bottomrule
\end{tabular}
\end{table}

\section{Kernel Robustness}
\label{app:kernel_robustness}
\subsection{Story plot-synopsis kernel}
\label{app:story_plot_synopsis}
The story robustness analysis evaluates whether below-parity crowding persists when stories are represented by one-sentence plot synopses rather than full prose. This kernel reduces the influence of style, wording, and surface phrasing, and instead measures convergence in narrative content.

\begin{figure}[h]
\centering
\includegraphics[width=1\linewidth]{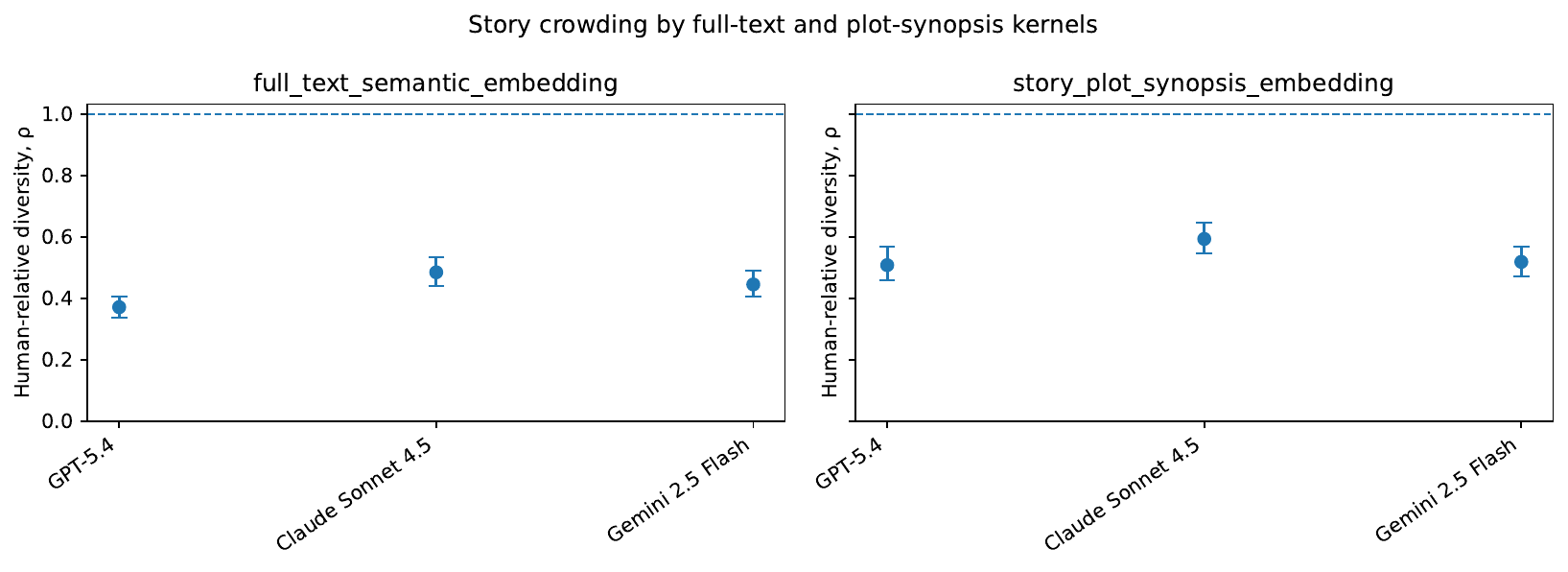}
\caption{
\textbf{Story crowding under full-text and plot-synopsis kernels.}
Points show task-family estimates of $\widehat{\rho}$ for each model; bars show 95\% bootstrap intervals. The dashed line marks $\rho=1$, the no-excess-crowding condition. All models remain below parity under the plot-synopsis kernel, indicating that story crowding persists when similarity is computed over narrative content rather than full prose.
}
\label{fig:story_kernel_comparison}
\end{figure}

Figure~\ref{fig:story_kernel_comparison} shows that the plot-synopsis kernel reduces the magnitude of the story deficit but does not remove it. GPT-5.4 moves from $\widehat{\rho}=0.372$ under full-text similarity to $0.509$ under plot-synopsis similarity, Claude Sonnet 4.5 moves from $0.485$ to $0.594$, and Gemini 2.5 Flash moves from $0.446$ to $0.519$. Thus, part of the full-text story effect reflects prose-level convergence, but the below-parity pattern remains after reducing stories to narrative synopses.

To assess finite-sample stability under this robustness kernel, Figure~\ref{fig:story_plot_synopsis_rarefaction} reports the task-level rarefaction curves for plot-synopsis crowding. The model curves stabilize within the sampled range, indicating that the narrative-level crowding estimates are not driven by the final few model-only generations.
\begin{figure}[h]
\centering
\includegraphics[width=0.7\linewidth]{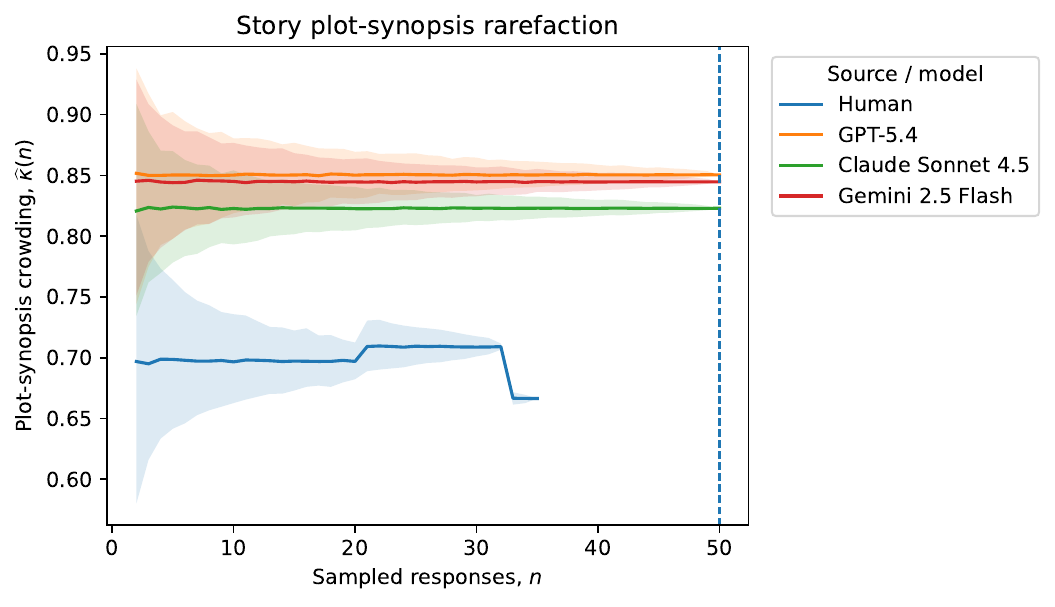}
\caption{
\textbf{Rarefaction curve for story plot-synopsis crowding.}
Curves show mean plot-synopsis crowding $\widehat{\kappa}(n)$ as a function of sampled responses $n$, aggregated across story prompts. The dashed vertical line marks $n=50$. The curves assess whether the narrative-level crowding estimates are stable with the available model-only sample size.
}
\label{fig:story_plot_synopsis_rarefaction}
\end{figure}

\subsection{Slogan lexical-template kernels}
\label{app:slogan_lexical}
Figure~\ref{fig:slogan_kernel_comparison} shows the human relative diversity of slogans under lexical template kernels. All models remain below human parity.

\begin{figure}[h]
\centering
\includegraphics[width=\linewidth]{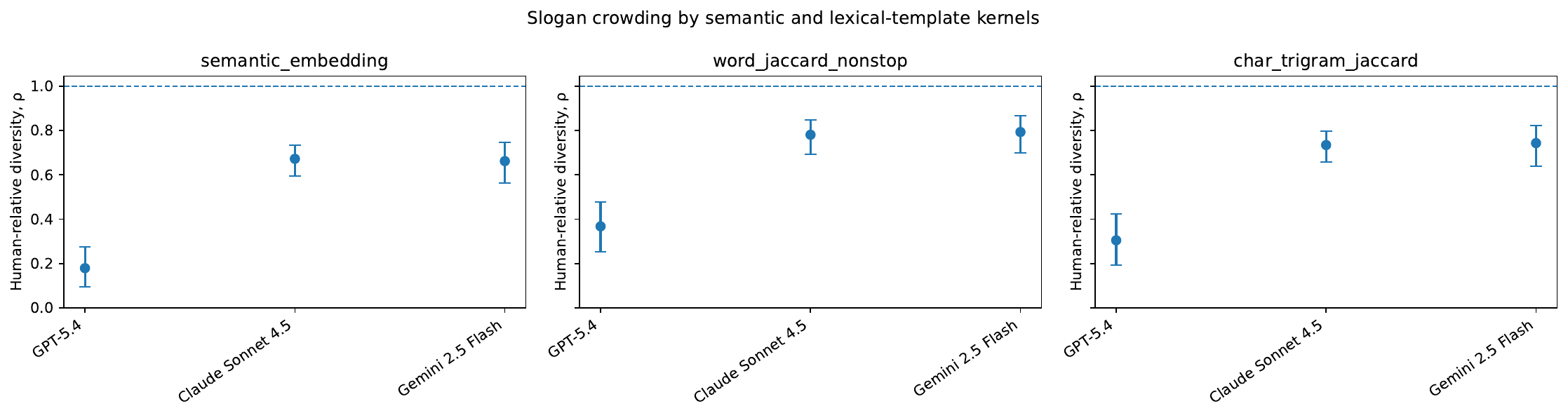}
\caption{
\textbf{Slogan crowding under semantic and lexical-template kernels.}
Points show $\widehat{\rho}$ for each model; bars show 95\% bootstrap intervals. The dashed line marks $\rho=1$, the no-excess-crowding condition. All models remain below parity under semantic, word-overlap, and character-trigram kernels.
}
\label{fig:slogan_kernel_comparison}
\end{figure}

Figures~\ref{fig:slogan_word_rarefaction} and \ref{fig:slogan_trigram_rarefaction} show rarefaction curves for the slogan lexical-template kernels. The curves stabilize by $n=50$, indicating that the lexical-template estimates are not driven by finite model-only sample size.

\begin{figure}[h]
\centering
\includegraphics[width=0.78\linewidth]{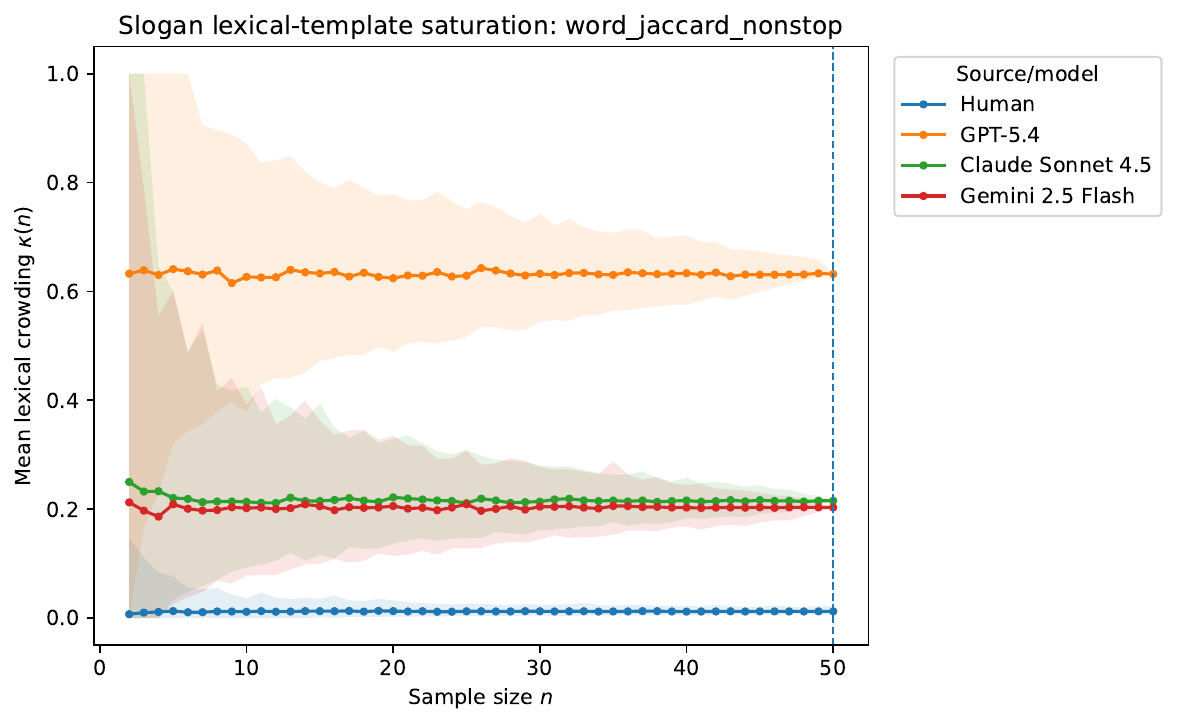}
\caption{
\textbf{Rarefaction curve for slogan non-stopword Jaccard crowding.}
Curves show mean lexical crowding $\widehat{\kappa}(n)$ as a function of sampled slogans $n$. The dashed vertical line marks $n=50$.
}
\label{fig:slogan_word_rarefaction}
\end{figure}

\begin{figure}[h]
\centering
\includegraphics[width=0.78\linewidth]{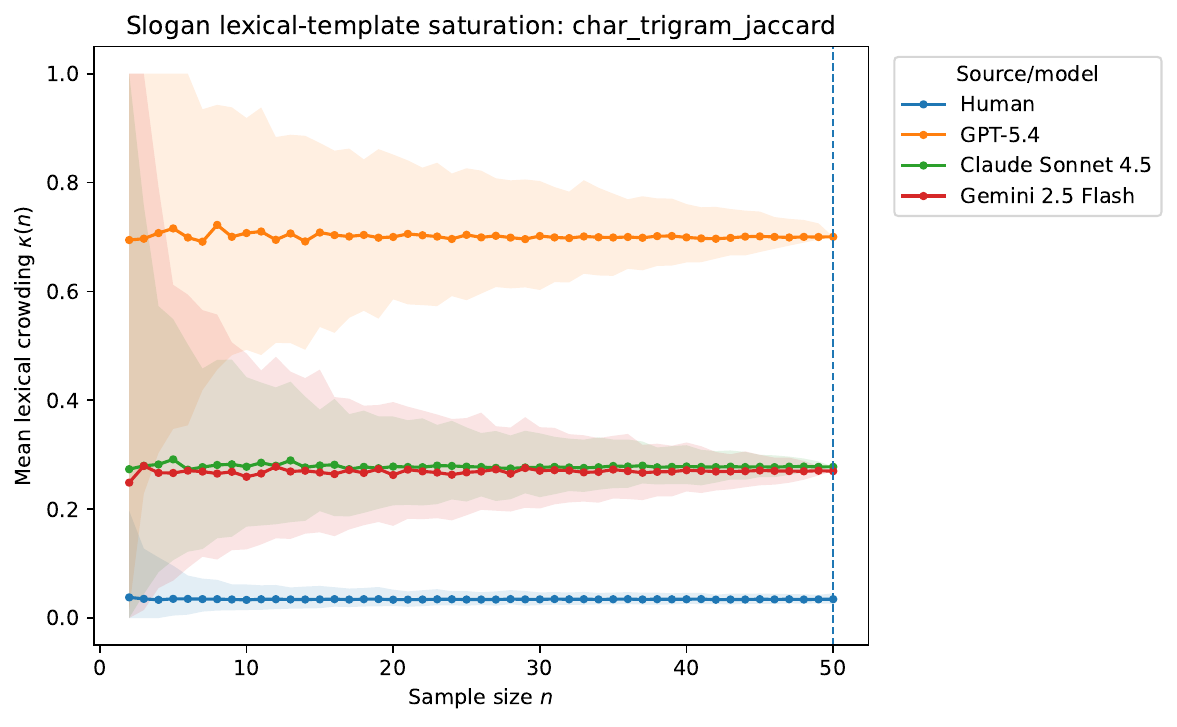}
\caption{
\textbf{Rarefaction curve for slogan character-trigram crowding.}
Curves show mean lexical crowding $\widehat{\kappa}(n)$ as a function of sampled slogans $n$. The dashed vertical line marks $n=50$.
}
\label{fig:slogan_trigram_rarefaction}
\end{figure}

\subsection{AUT concept-bucket kernel}
\label{app:aut_bucket_kernel}
The AUT concept-bucket kernel evaluates whether two responses express the same underlying use concept. Human and model responses are assigned to object-specific buckets, and the kernel equals one when two responses fall in the same bucket and zero otherwise. Thus, under this kernel, $\widehat{\kappa}$ is the probability that two sampled responses reuse the same concept.

\begin{figure}[h]
\centering
\includegraphics[width=0.78\linewidth]{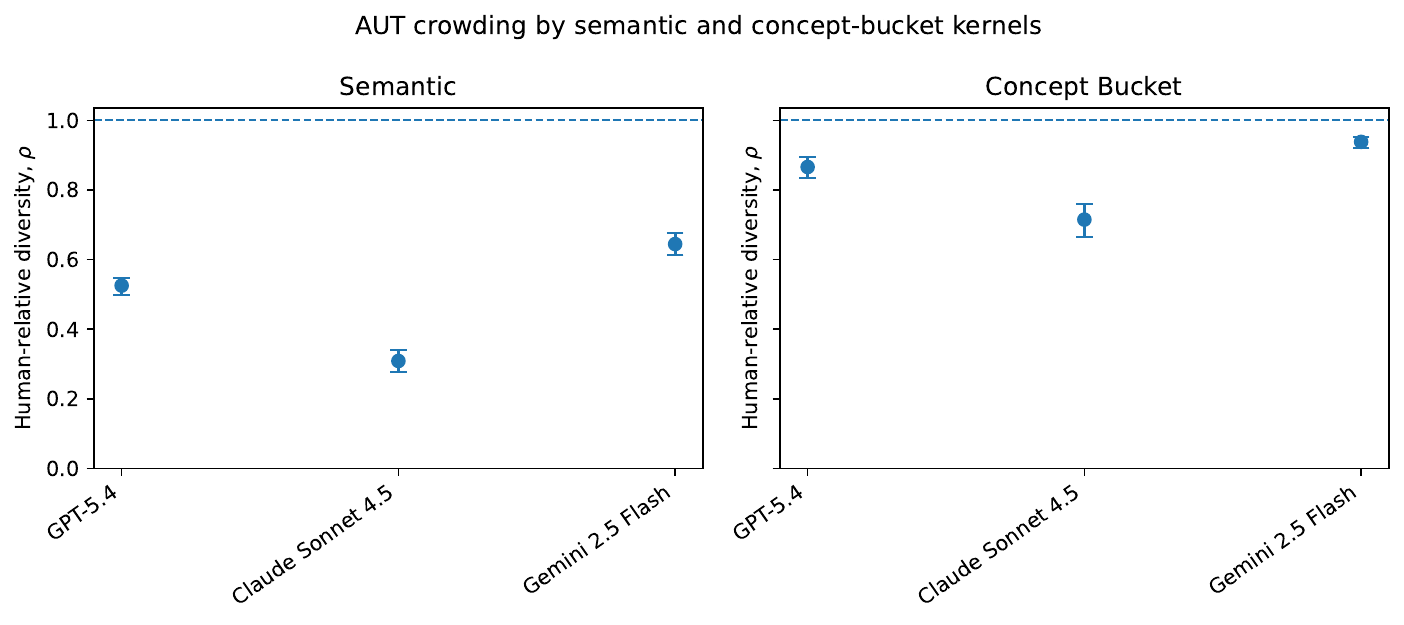}
\caption{
\textbf{AUT crowding under semantic and concept-bucket kernels.}
Points show task-family estimates of $\widehat{\rho}$ for each model; bars show 95\% bootstrap intervals. The dashed line marks $\rho=1$, the no-excess-crowding condition. All models remain below parity under the concept-bucket kernel, indicating excess concept reuse relative to the matched human baseline.
}
\label{fig:aut_bucket_kernel_comparison}
\end{figure}

Figure~\ref{fig:aut_bucket_kernel_comparison} compares the primary semantic kernel with the concept-bucket kernel. Under concept-bucket co-membership, all three models remain below parity: GPT-5.4 has $\widehat{\rho}=0.866$ with 95\% CI $[0.833,0.894]$, Claude Sonnet 4.5 has $\widehat{\rho}=0.715$ with 95\% CI $[0.665,0.759]$, and Gemini 2.5 Flash has $\widehat{\rho}=0.938$ with 95\% CI $[0.920,0.953]$. Figure~\ref{fig:aut_bucket_rarefaction} shows the rarefaction curves.

\begin{figure}[h]
\centering
\includegraphics[width=0.78\linewidth]{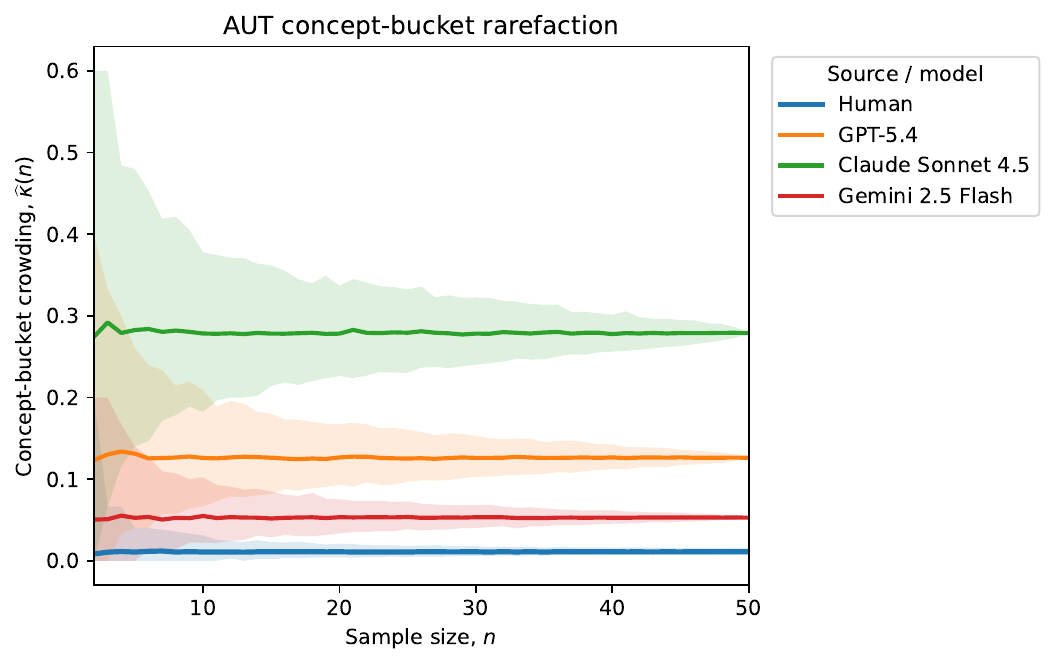}
\caption{
\textbf{Rarefaction curve for AUT concept-bucket crowding.}
Curves show mean concept-bucket crowding $\widehat{\kappa}(n)$ as a function of sampled responses $n$, aggregated across AUT objects. The dashed vertical line marks $n=50$. The curves assess whether the concept-level crowding estimates are stable with the available model-only sample size.
}
\label{fig:aut_bucket_rarefaction}
\end{figure}

\section{Protocol Diagnostics}
\label{app:protocol_diagnostics}
\subsection{Persona-mixture protocol comparison}
\label{app:persona_mix}
Table~\ref{tab:persona_mix_vs_main} reports the direct comparison between the neutral main protocol and persona-mixture prompting at $T=1.0$. Difference intervals are bootstrap intervals for $\widehat{\rho}^{persona}-\widehat{\rho}^{main}$.

\begin{table}[h]
\centering
\caption{
\textbf{Persona mixture versus neutral main prompting at fixed temperature.}
Both protocols use $T=1.0$. Positive $\Delta\widehat{\rho}$ indicates improved human-relative diversity under persona mixture.
}
\label{tab:persona_mix_vs_main}
\small
\begin{tabular}{llrrrr}
\toprule
Task & Model & Main $\widehat{\rho}$ & Persona $\widehat{\rho}$ & $\Delta\widehat{\rho}$ & 95\% CI for $\Delta\widehat{\rho}$ \\
\midrule
AUT & Claude Sonnet 4.5 & 0.309 & 0.739 & 0.430 & [0.388, 0.474] \\
AUT & GPT-5.4 & 0.526 & 0.507 & -0.018 & [-0.053, 0.017] \\
AUT & Gemini 2.5 Flash & 0.645 & 0.766 & 0.121 & [0.080, 0.164] \\
Slogans & Claude Sonnet 4.5 & 0.673 & 0.960 & 0.287 & [0.201, 0.382] \\
Slogans & GPT-5.4 & 0.179 & 0.927 & 0.749 & [0.650, 0.846] \\
Slogans & Gemini 2.5 Flash & 0.660 & 0.926 & 0.265 & [0.153, 0.391] \\
Stories & Claude Sonnet 4.5 & 0.486 & 0.699 & 0.213 & [0.145, 0.284] \\
Stories & GPT-5.4 & 0.372 & 0.616 & 0.244 & [0.189, 0.300] \\
Stories & Gemini 2.5 Flash & 0.446 & 0.724 & 0.278 & [0.212, 0.345] \\
\bottomrule
\end{tabular}
\end{table}

Figure~\ref{fig:persona_bcrit_actionability} shows how persona-mixture prompting lowers the critical private benefit required for rational adoption.

\begin{figure}[h]
\centering
\includegraphics[width=\linewidth]{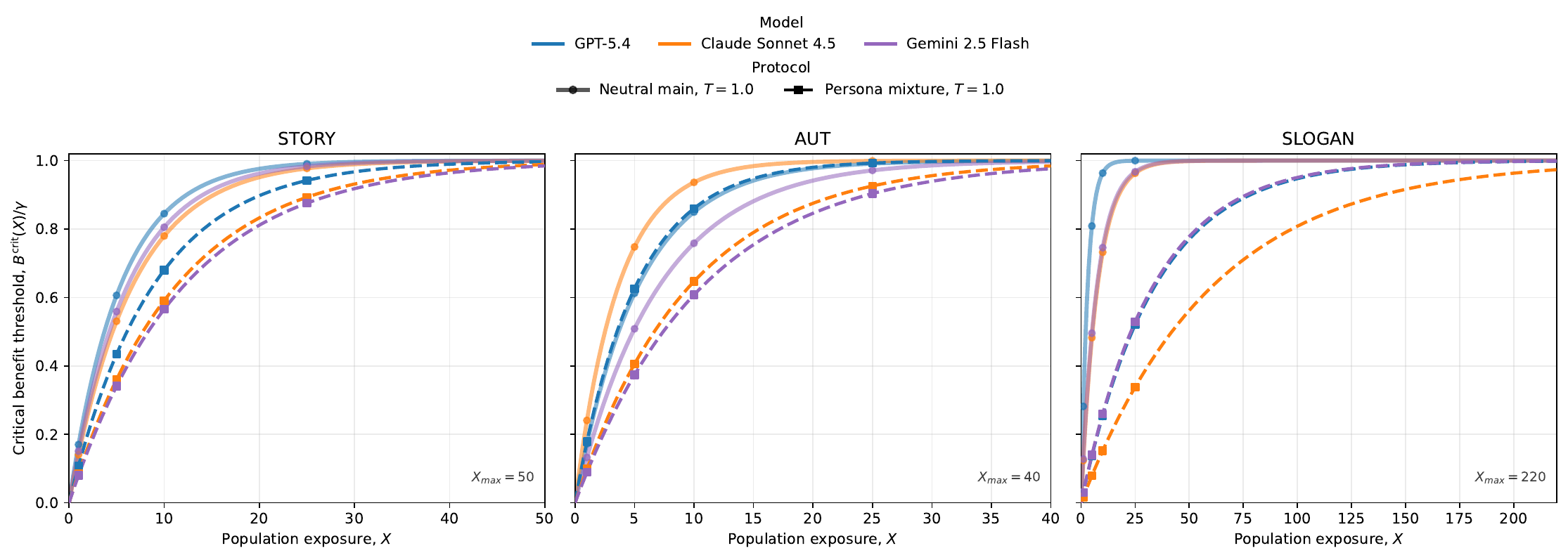}
\caption{
\textbf{Persona-mixture prompting lowers the critical private benefit required for rational adoption.}
Solid curves show the neutral main protocol at $T=1.0$; dashed curves show the persona-mixture protocol at $T=1.0$. Curves plot $B^{\mathrm{crit}}(X)/\gamma = 1-\exp(-X\widehat{\Delta})$. Lower curves indicate weaker congestion externalities.
}
\label{fig:persona_bcrit_actionability}
\end{figure}

\subsection{Temperature-grid monotonicity}
\label{app:temperature_grid}
Figure~\ref{fig:temperature_grid_rho} reports human-relative diversity over the tested temperature grid. The neutral $T=1.0$ point is downsampled to match the robustness sample size, making it comparable to the non-main temperature settings.

\begin{figure}[h]
\centering
\includegraphics[width=\linewidth]{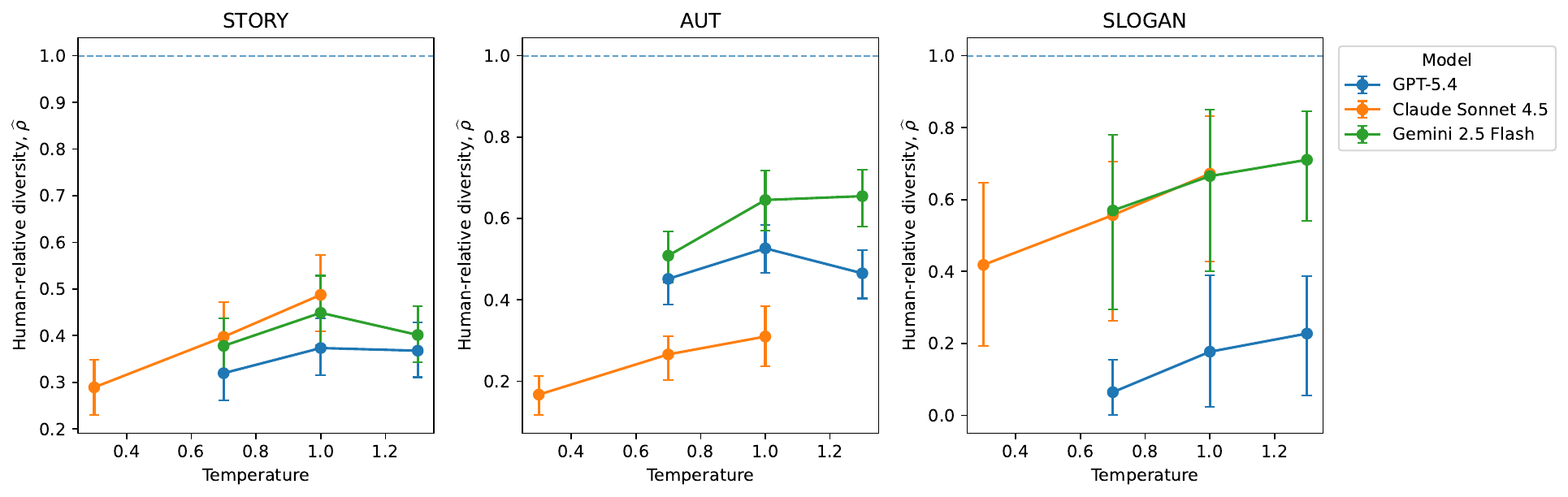}
\caption{
\textbf{Temperature-grid effects on human-relative diversity.}
Points show task-level $\widehat{\rho}$ under the primary semantic kernel across available temperatures. Error bars show 95\% bootstrap intervals. Higher temperature increases $\widehat{\rho}$ from the lowest to highest tested temperature in all nine model-task combinations, although strict monotonicity holds in six of nine.
}
\label{fig:temperature_grid_rho}
\end{figure}

Figure~\ref{fig:temperature_grid_delta} reports the corresponding effect on excess crowding.

\begin{figure}[h]
\centering
\includegraphics[width=\linewidth]{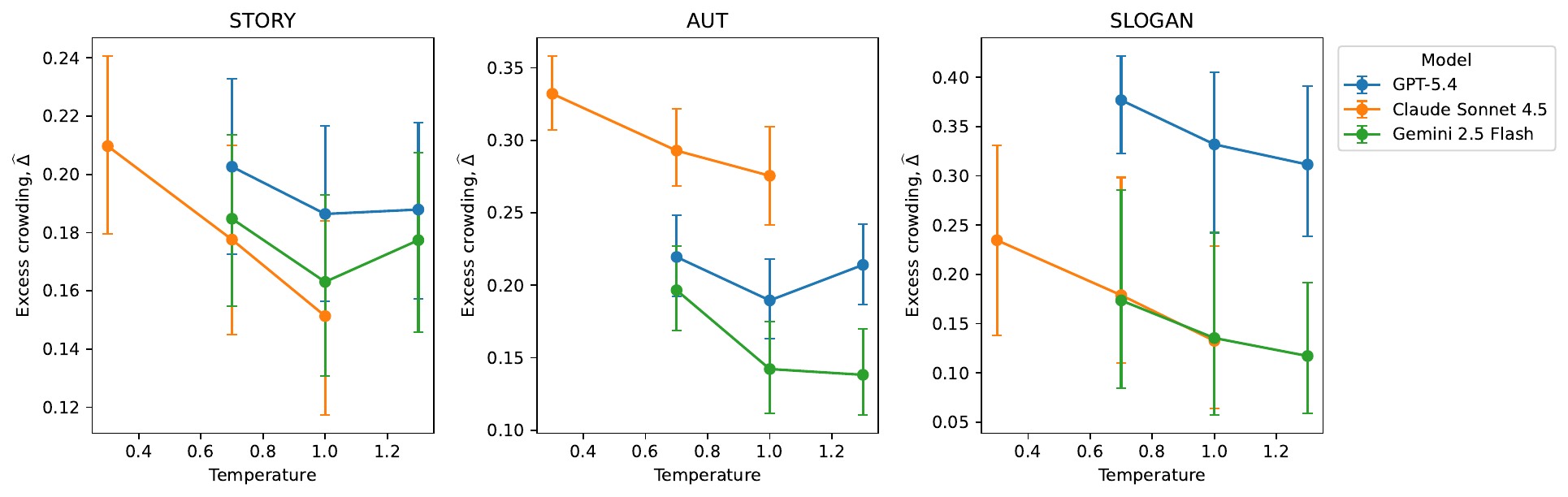}
\caption{
\textbf{Temperature-grid effects on excess crowding.}
Points show task-level $\widehat{\Delta}$ under the primary semantic kernel across available temperatures. Error bars show 95\% bootstrap intervals. Higher temperature decreases $\widehat{\Delta}$ from the lowest to highest tested temperature in all nine model-task combinations.
}
\label{fig:temperature_grid_delta}
\end{figure}

Table~\ref{tab:temperature_monotonicity} reports monotonic associations between temperature and crowding metrics. Because each model-task pair has only three temperature points, these tests are descriptive diagnostics rather than large-sample trend tests.

\begin{table}[h]
\centering
\caption{
\textbf{Temperature-grid monotonicity diagnostics.}
Positive $\rho_s(T,\widehat{\rho})$ indicates that higher temperature is associated with higher human-relative diversity. Negative $\rho_s(T,\widehat{\Delta})$ indicates that higher temperature is associated with lower excess crowding.
}
\label{tab:temperature_monotonicity}
\small
\begin{tabular}{llrrrr}
\toprule
Task & Model & $\Delta\widehat{\rho}$ & $\rho_s(T,\widehat{\rho})$ & $\Delta\widehat{\Delta}$ & $\rho_s(T,\widehat{\Delta})$ \\
\midrule
AUT & Claude Sonnet 4.5 & 0.143 & 1.0 & -0.057 & -1.0 \\
AUT & GPT-5.4 & 0.014 & 0.5 & -0.005 & -0.5 \\
AUT & Gemini 2.5 Flash & 0.146 & 1.0 & -0.058 & -1.0 \\
Slogans & Claude Sonnet 4.5 & 0.254 & 1.0 & -0.102 & -1.0 \\
Slogans & GPT-5.4 & 0.163 & 1.0 & -0.065 & -1.0 \\
Slogans & Gemini 2.5 Flash & 0.141 & 1.0 & -0.056 & -1.0 \\
Stories & Claude Sonnet 4.5 & 0.198 & 1.0 & -0.058 & -1.0 \\
Stories & GPT-5.4 & 0.048 & 0.5 & -0.015 & -0.5 \\
Stories & Gemini 2.5 Flash & 0.023 & 0.5 & -0.007 & -0.5 \\
\bottomrule
\end{tabular}
\end{table}

\subsection{Best observed protocol and critical-benefit curves}
\label{app:best_protocol}
The main protocol analysis isolates persona-mixture prompting at fixed $T=1.0$. As an additional diagnostic, we identify the best observed protocol among the tested temperature and persona-mixture settings for each model-task pair. Figure~\ref{fig:best_protocol_bcrit} compares the neutral main critical-benefit curve with the best observed protocol curve.

\begin{figure}[h]
\centering
\includegraphics[width=\linewidth]{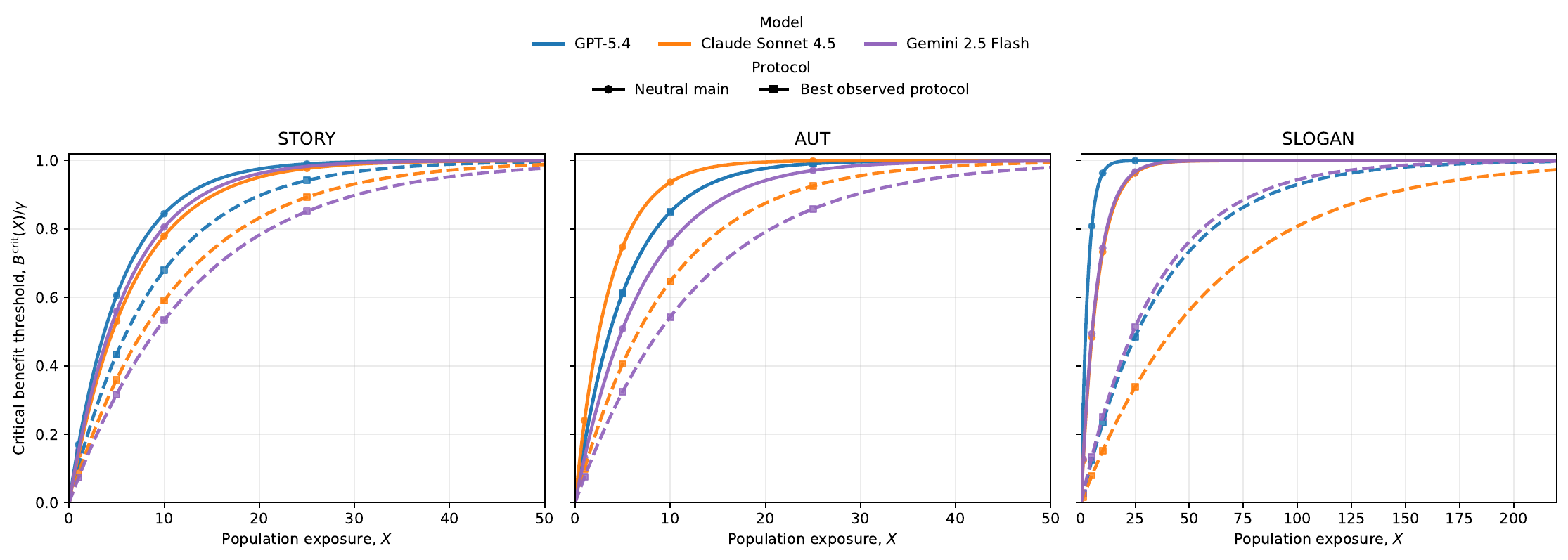}
\caption{
\textbf{Critical benefit curves under neutral main and best observed protocols.}
For each model-task pair, the solid curve shows the neutral main protocol and the dashed curve shows the best observed protocol among the tested temperature and persona-mixture settings. Curves plot $B^{\mathrm{crit}}(X)/\gamma=1-\exp(-X\widehat{\Delta})$. This figure is descriptive: it shows the best protocol found in the tested grid, not an optimized global protocol.
}
\label{fig:best_protocol_bcrit}
\end{figure}


\clearpage
\bibliographystyle{plain}
\bibliography{references}

\end{document}